\documentclass{article}

\usepackage{microtype}
\usepackage{graphicx}
\usepackage{subcaption}
\usepackage{booktabs} 

\usepackage{hyperref}

\usepackage[accepted]{icml2024}

\usepackage{amsmath}
\usepackage{amssymb}
\usepackage{mathtools}
\usepackage{amsthm}

\usepackage[capitalize,noabbrev]{cleveref}

\theoremstyle{plain}
\newtheorem{theorem}{Theorem}[section]
\newtheorem{proposition}[theorem]{Proposition}
\newtheorem{lemma}[theorem]{Lemma}
\newtheorem{corollary}[theorem]{Corollary}
\theoremstyle{definition}
\newtheorem{definition}[theorem]{Definition}

\theoremstyle{remark}
\newtheorem{remark}[theorem]{Remark}

\usepackage[textsize=tiny]{todonotes}

\newcommand{\colvector}[1]{\underset{\big|}{\overset{\big|}{#1}}}
\newcommand{\rowvector}[1]{\rule[.5ex]{1em}{.55pt} \, #1 \, \rule[.5ex]{1em}{.55pt}}

\icmltitlerunning{Sign Rank Limitations for Inner Product Graph Decoders}

\begin{document}

\twocolumn[
\icmltitle{Sign Rank Limitations for Inner Product Graph Decoders}



\icmlsetsymbol{equal}{*}

\begin{icmlauthorlist}
\icmlauthor{Su Hyeong Lee}{yyy}
\icmlauthor{Qingqi Zhang}{yyy}
\icmlauthor{Risi Kondor}{yyy,comp}
\end{icmlauthorlist}

\icmlaffiliation{yyy}{Department of Statistics, University of Chicago}
\icmlaffiliation{comp}{Department of Computer Science, University of Chicago}

\icmlcorrespondingauthor{Su Hyeong Lee}{sulee@uchicago.edu}

\icmlkeywords{Machine Learning, ICML}

\vskip 0.3in
]



\printAffiliationsAndNotice{} 

\begin{abstract}
Inner product-based decoders are among the most influential frameworks used to extract meaningful data from latent embeddings. However, such decoders have shown limitations in representation capacity in numerous works within the literature, which have been particularly notable in graph reconstruction problems. In this paper, we provide the first theoretical elucidation of this pervasive phenomenon in graph data, and suggest straightforward modifications to circumvent this issue without deviating from the inner product framework.
\end{abstract}
\section{Introduction}

\subsection{Background}
Real world graph data pertinent to scientific applications often reside within high-dimensional, non-Euclidean spaces where a succinct representation remains elusive (\citealt{hyperbolicembedding}; \citealt{Bronstein}; \citealt{mixedembedding}). For instance, in biochemistry or molecular drug design, the multinodal nature of the molecular backbone of interest necessitates devising sophisticated statistical techniques to capture essential node-relational attributes of the data inside of a low-dimensional latent space (\citealt{transe}; \citealt{molrep1}; \citealt{molrep4}).

Deep learning has achieved great success in distilling compact latent representations from complex graph data which preserves their structural integrity (\citealt{dl2}; \citealt{dl3}). Once the latent vectors have been synthesized, they can be channeled into trainable decoders to perform various tasks such as graph/link reconstruction (\citealt{VGAE}; \citealt{node2vec}; \citealt{recentreconstruct}) or node clustering (\citealt{cluster1}; \citealt{cluster2}; \citealt{cluster3}). The literature has seen the emergence of a diverse array of decoders for shallow embedding schemes such as Laplacian Eigenmaps (\citealt{lapeigmap}), DeepWalk (\citealt{deepwalk}), and node2vec (\citealt{node2vec}). In particular, many decoders start by extracting an ideal representation of the graph input by taking inner products between the latent embeddings of each node (\citealt{dot1}; \citealt{dot2}; \citealt{dot3}). Inner product-based algorithms also play a critical role in many other fields of machine learning, such as in computing the similarity in word vector representations in Skip-Gram Negative Sampling (\citealt{word2vec}; \citealt{skipgramnegsamp}).

Since its conception and popularization by \citealp{VGAE} for \textit{Graph Neural Networks (GNNs)}, inner product decoders have empirically been demonstrated to have limitations on expressiveness (\citealt{graphRepLearnbook}). This impediment appears to persist universally whenever inner product decoders are deployed, for instance when latent representations are generated probabilistically using adversarial training to extract maximal node-relational information (\citealt{adversarialGVAE}). To our knowledge, no theoretical studies in the literature have elucidated this pervasive phenomenon. 

In this paper, we formalize the notion of low dimensionality in latent representations using the \textit{sign rank} (\citealt{signrank1}; \citealt{signrank2}), and provide a deterministic bound on the minimum latent feature dimension necessary for a faithful reconstruction of the embedded graph. Furthermore, we present examples of pedagogical graph structures for which complexifying the latent space permits significantly lower dimensional latent encodings to be used. Afterwards, we design a decoding architecture which drastically expands the representation capacity of inner product decoders that subsume the expressivity of the aforementioned complex GNN. 

Section $1$ discusses the motivation, related work, and our contributions. Section $2$ formulates the problem and generates low-rank graphs. Section $3$ categorically motivates and proves a lower bound on the minimum latent dimension necessary for a faithful graph reconstruction. Section $4$ details an augmentation scheme to boost decoder representation capacity. Section $5$ presents the experiments, and the conclusion and possible extensions are given in Section $6$.

\subsection{Problem motivation}
The adjacency $\mathbf{A} \in \{0,1\}^{N\times N}$ of an unweighted, undirected graph is a Hermitian matrix encoding binary information, where the entry $A_{ij} = 1$ indicates a connection between nodes $i$ and $j$. In a binary encoding, the values of the two objects are insignificant. While using $0$ or $1$ to represent node relations may appear innocuous, this translates to nontrivial consequences down the line, such as the adjacency matrices of sparse graph structures failing to be low-rank.

This motivates the scheme of training a neural net to discover a latent embedding while respecting the insignificance of the binary values. In such an approach, the GNN can be fit to a set of Hermitian matrices where positive entries indicate a connection while non-positive entries indicate a disconnection. This allows the net to search for a low-dimensional representation from an equivalence class of adjacencies, where each class holds uncountably infinite cardinality. 
\begin{figure}[h]
\vspace{.1in}
\includegraphics[width=0.46\textwidth]{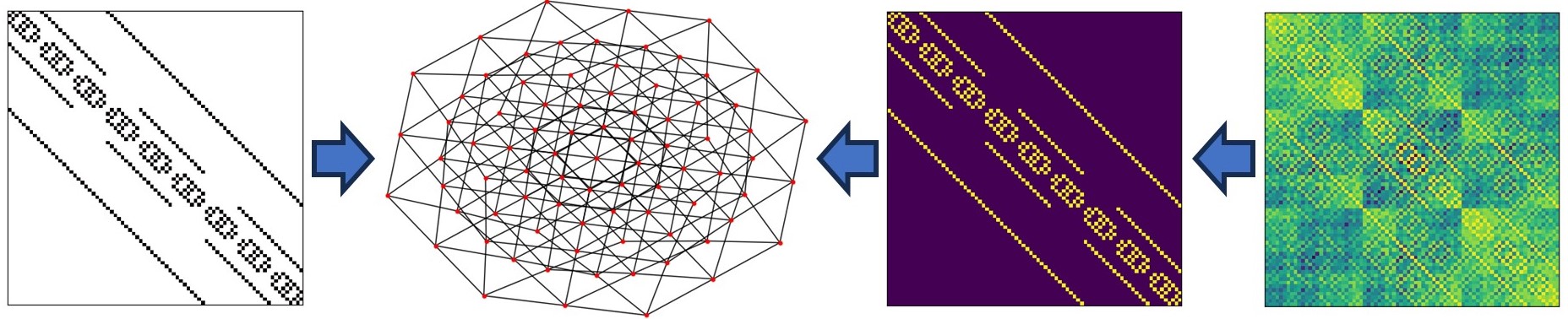}
\vspace{.1in}
\caption{The adjacency $\mathbf{A}$ of the $4$-dimensional $3\times 3\times 3\times 3$ grid graph shown on the left has matrix rank $62$, whereas the rightmost matrix has matrix rank $6$. Both matrices belong to the same equivalence class under the $\operatorname{sign}$ mapping applied elementwise.}
\end{figure}

Under this approach, we study the classes of graphs representable by low-rank embeddings, giving a complete classification of all graphs up to rank $3$. We also present a variety of different strategies for proving that a graph cannot be represented by a low-rank embedding, and introduce \textit{cutoffs} (Section~\ref{cutoff}) to drastically enhance the performance of inner product decoding architectures. Cutoffs are trainable linear weights in the decoder that can be selected to scale proportionally to the number of latent \textit{feature} dimensions, which maintains economical cost as the number of nodes grows high-dimensional.

\subsection{Notable related work}

Given a feature matrix $\mathbf{X} \in \mathbb{R}^{N\times d}$ and an adjacency matrix $\mathbf{A} \in \mathbb{R}^{N\times N}$, a \textit{Graph Convolutional Network (GCN)} can be utilized to learn a latent mapping $\mathbf{X}\mapsto \mathbf{Z} \in \mathbb{R}^{N\times f}$, where $f$ is the latent dimension. Arguably, the most famous example of a GCN employing an inner product decoder is given by the seminal paper by \citealp{VGAE}, which details a straightforward \textit{Graph Autoencoder (GAE)} architecture. In their inference model, the latent matrix $\mathbf{Z}$ is formed by message passing through the two layer network $\operatorname{GCN}(\mathbf{X}, \mathbf{A})=\tilde{\mathbf{A}} \operatorname{ReLU}(\tilde{\mathbf{A}} \mathbf{X} \mathbf{W}_0) \mathbf{W}_1$. 

After computing the latent embedding $\mathbf{Z}$ with $i$-th \textit{row vector} $\mathbf{w}_i$, a stochastic decoding model is formed by taking inner products, 
\begin{equation}\label{innerproductdecoder}
p\left(A_{i j}=1 \left| \mathbf{Z}\right.\right)=\sigma\left(\mathbf{w}_i^\top \mathbf{w}_j\right),    
\end{equation}
where $\sigma$ denotes the sigmoid. In essence, the decoding is done by interpreting entries of $\sigma(\mathbf{ZZ}^\top)$ as Bernoulli probabilities of successful connections between the nodes, which forms a distribution over computed adjacency matrices $\hat{\mathbf{A}}$. Additional details are given in Appendix $1$.

During this process, the latent dimension $f$ is tuned empirically as a hyperparameter until the preferred model behavior is observed. As the node count of the embedded graph grows, a significantly lower value of $f$ which infers or reproduces the adjacency $\mathbf{A}$ is desirable due to prohibitive computational cost. 

\subsection{Our Contributions}

The sign rank is typically studied using VC dimensions, where previous works have attempted to form subexponential time algorithms for computing the sign rank (\citealt{signrank3}; \citealt{signrank2}). In this work, we dislocate sign rank considerations from VC dimension by introducing cutoffs that transmute planar classifiers into conic classifiers. Furthermore, we show that conic classification can be learned using a very low-rank ensembling operation, which permits the classifier to attain pairwise distinct classification thresholds to establish node relational strength in graph data. 

By linking the sign rank, graph reconstructions, and non-linear classifiers, we present a powerful new theory on graph representations which explains why the standard real inner product decoder is a provably weak choice for distilling graph data despite their ubiquity in other areas of machine learning. Finally, we provide concrete examples illustrating how the algebraic latent substructure induced by the complexification of the neural net allows the decoder to sidestep the aforementioned limitations entirely. To our knowledge, this paper also presents the first use of an ML architecture to compute what may be realized as an extremely tight upper bound to the sign rank.

\section{Classifying Low Rank Graphs}

We define the function $\operatorname{sign}: \mathbb{R} \to \{+,- \}$ to take the negative value in $\mathbb{R}_{\le0}$ and positive otherwise. Assigning the positive value to $0$ does not confer any substantive changes to the results presented in this paper. The following definition is written for a complex field for clarity.

\begin{definition}
Let $\mathbf{A} \in \mathbb{R}^{N \times N}$ and $F = \mathbb{C}$. Then, the \textit{complex sign rank} of $\mathbf{A}$ is the minimal $f$ such that there exists $\mathbf{Z} \in F^{N\times f}$ with columns $\mathbf{z}_i$ satisfying 

\resizebox{0.47\textwidth}{!}{$
\operatorname{sign}(\mathbf{A}) = \operatorname{sign}\left(\mathfrak{Re}\left(\left[\colvector{\mathbf{z}_1} \colvector{\mathbf{z}_2} \dots \colvector{\mathbf{z}_f} \right] \times   \left[
\begin{array}{cc}
\rowvector{\mathbf{z}_1^\top} \\
\rowvector{\mathbf{z}_2^\top} \\ 
\vdots \\ 
\rowvector{\mathbf{z}_f^\top}
\end{array}
\right] \right)\right) .$}

The \textit{real sign rank} is given by replacing $F = \mathbb{R}$.
\end{definition}

In this paper, rank is taken by default to mean real sign rank unless noted otherwise. This is consistent with the current paradigm of working in the real field for graph machine learning.

\subsection{Graphs of rank $1$}\label{rank1graphs}

As we have relaxed the requirement that the adjacency matrix must contain binary values, it is natural to question how much more information can be embedded in the latent representation $\mathbf{Z}$. It is easy to classify all graphs of rank $1$ with $N$ nodes by taking conjugate multiples of $\mathbf{z} \in \{0,1\}^N \cup \{-1,1\}^N $, $\mathbf{z}\mathbf{z}^\top$. By a straightforward counting argument, this cannot possibly reproduce all graph structures with $N$ nodes.

\subsection{Classification of graphs of ranks $2$ or $3$}
We now work towards presenting a complete classification of rank $2$ graphs by providing a framework capable of generating all permissible sign combinations. To avoid trivial (and thus uninteresting) counterexamples in the theory, we limit the analysis to graphs that do not admit self-loops and ignore adjacency diagonals. 

\begin{figure}[h]
\vspace{.1in}
\includegraphics[width=0.46\textwidth]{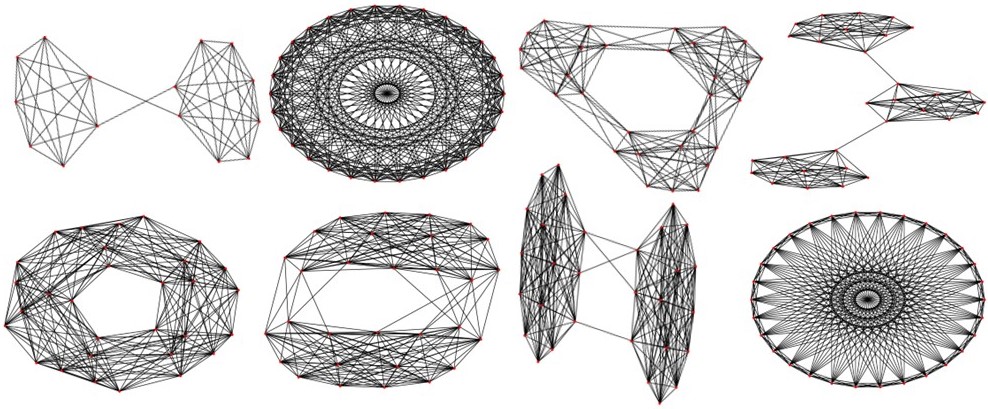}
\vspace{.1in}
\caption{Examples of rank $\le 2$ graphs.}
\label{rank2examples}
\end{figure}

The easiest approach to generating rank $2$ graphs that we could find was to take conjugate multiples of a real matrix function with two column dimensions. For $k_1, \dots, k_n \in \mathbb{R}$, let
\begin{equation}\label{choosethismatfunc}
\mathbf{Z} = \left(\begin{array}{cc}
 \cos(k_1  x)    &  \sin(k_1  x)\\
 \cos(k_2  x)     &  \sin(k_2  x) \\
\vdots   & \vdots   \\
\cos(k_N  x)  &  \sin(k_N  x) \\
\end{array} \right), \quad \tilde{\mathbf{A}} = \mathbf{ZZ}^\top.    
\end{equation}
The entries of the low rank representation $\tilde{\mathbf{A}}$ are expressible as a single trigonometric function, $\tilde{A}_{ij} = \cos ((k_i - k_j) x)$. Out of all matrix functions available, the reason for choosing~\eqref{choosethismatfunc} is suggested by Lemmas \ref{easylemma}, \ref{easylemma2}, which intuit that the expression should be capable of producing all sign combinations of rank $2$. The proofs are contained in Appendix $2$.
\begin{lemma}\label{easylemma}
Let $k_i : = \sqrt{p_i}$ where $p_1<p_2<\dots$ are any sequence of positive integer primes. Limiting the indices to the lower triangular portion $i>j$, the periods of $\tilde{A}_{ij}$ can never match. That is, there exists no $n_1,n_2 \in \mathbb{Z}_{\neq 0}$ such that $n_1 t_1 = n_2 t_2$, where $t_1,t_2$ are periods of $ \tilde{A}_{ij}, \tilde{A}_{i^\prime,j^\prime}$ for $\{i,j\} \neq \{i^\prime,j^\prime \}.$
\end{lemma}
Therefore, each entry of $\tilde{\mathbf{A}}$ is of the form $\cos(\theta x)$, where the $\theta$ are unique to every lower triangular entry. Note that $x = 0$ initializes all such entries to be $1$, and Lemma \ref{easylemma} guarantees that this will never happen again as we vary $x$ over the real line. However, the values of any two non-diagonal entries can become arbitrarily close to $1$ pairwise as we vary $x$ far afield. This is formally expressed by Lemma~\ref{easylemma2}, which possesses an elegant proof (Appendix $2.1$). 

\begin{lemma}\label{easylemma2}
Let $t_1$, $t_2$ be as in Lemma \ref{easylemma}. Then for any $\varepsilon>0$, there exists $n_1,n_2 \in \mathbb{Z}_{>0}$ such that $|n_1 t_1 - n_2 t_2| < \varepsilon.$
\end{lemma}
Taken together, Lemmas~\ref{easylemma} and~\ref{easylemma2} ensure that any two distinct lower triangular entries of $\mathbf{\tilde{A}}$ may never simultaneously take the value $1$ for any $x \neq 0$, but may become arbitrarily close to $1$ for large enough $x$. At this point, their relationship resets from the perspective of their periods. As we demand that $\varepsilon$ converges to $0^+$, it is expected that this reset happens farther afar in the real line. Due to the pairwise periodic regularity of the entries of $\mathbf{\tilde{A}}$ which emerges despite their perpetually imbalanced periods (Lemma~\ref{easylemma}), it is natural and intuitive that all possible sign combinations of rank $2$ should be taken as we vary $x \in \mathbb{R}$.

The following proposition, which admits a simple proof, formally confirms that this is indeed the case. 
\begin{proposition}\label{firstproposition}
Let $\mathbf{Z}$, $\tilde{\mathbf{A}}$ be as in~\eqref{choosethismatfunc}. Then, for any rank $2$ adjacency $\mathbf{A}$ representing a connected graph, there exists $\left(k_1, \dots, k_n\right)$ such that $\operatorname{sign}(\mathbf{A}) = \operatorname{sign}(\tilde{\mathbf{A}})$.
\end{proposition}
\begin{proof}
As $\mathbf{A}$ is rank $2$, there exists $\mathbf{Z_A} \in \mathbb{R}^{N\times 2}$ such that $\operatorname{sign}(\mathbf{A}) = \operatorname{sign}(\mathbf{Z_AZ_A}^\top)$. Denoting the $i$-th row of $\mathbf{Z}_\mathbf{A}$ as $\mathbf{w}_i$, note that $\mathbf{w}_i \neq \overline{0}$ for all $i$ (otherwise the $i$-th node is disconnected). Scaling the length of $\mathbf{w}_i$ by a positive scalar does not impact the sign of $\tilde{\mathbf{A}}_{ij} = \langle \mathbf{w}_i, \mathbf{w}_j\rangle$, which allows us to normalize $||\mathbf{w}_i||_{\ell_2} = 1.$ Finally, note that for any point $(a,b) \in S^1$ and $x \neq 0$ arbitrarily fixed, there exists a $k$ such that $(\cos(kx),\sin(kx)) = (a,b)$. 
\end{proof}
An analogous geometric argument on the sphere $S^2$ generalizes Proposition~\ref{firstproposition} to rank $3$ graphs.
\begin{proposition}\label{firstpropgeneralized}
For $k_1, k_1^\prime, \dots, k_N,k_N^\prime \in \mathbb{R}$, let

\resizebox{0.48\textwidth}{!}{$
\mathbf{Z} = \left(\begin{array}{ccc}
 \cos(k_1  x) \sin(k_1^\prime x)   &  \cos(k_1  x) \cos(k_1^\prime x)   &  \sin(k_1  x)\\
\cos(k_2  x) \sin(k_2^\prime x)    &  \cos(k_2  x) \cos(k_2^\prime x)    &  \sin(k_2  x) \\
\vdots   & \vdots   & \vdots   \\
\cos(k_N  x) \sin(k_N^\prime x) & \cos(k_N  x)\cos(k_N^\prime x)  &  \sin(k_N  x) \\
\end{array} \right)
$}

and $\tilde{\mathbf{A}} = \mathbf{ZZ}^\top$. Then for any connected rank $3$ adjacency $\mathbf{A}$, there exists $k_i,k_i^\prime$ such that $\operatorname{sign}(\mathbf{A}) = \operatorname{sign}(\tilde{\mathbf{A}})$.
\end{proposition}

\begin{figure}[h]
\vspace{.1in}
\includegraphics[width=0.46\textwidth]{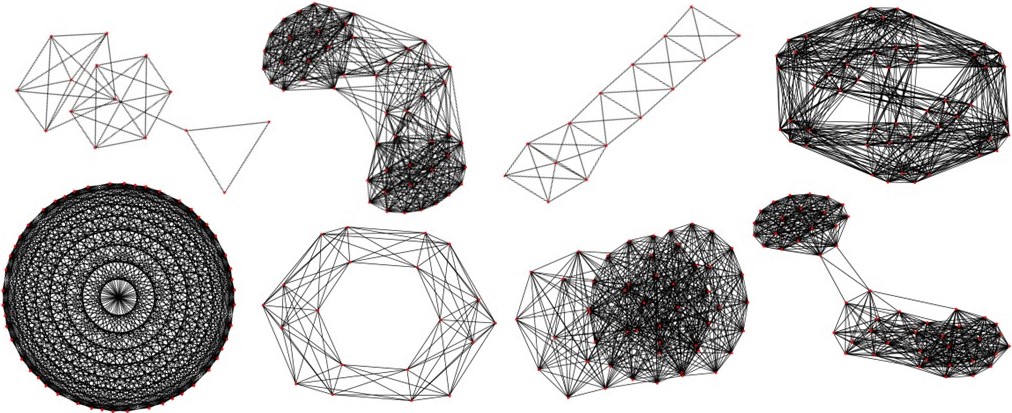}
\vspace{.1in}
\caption{Examples of rank $\le 3$ graphs.}
\label{rank3examples}
\end{figure}

\section{Lower Bounds for Graph Rank}\label{lowerboundinggraphrank}
The discussion in Section~\ref{rank1graphs} leads naturally to a similar question: \textit{Can all graphs be embedded in a rank $2$ representation}? We give an answer to the negative by constructing a counterexample. Due to the mechanical nature of the construction, we have elected to detail full proofs in the appendix while providing only key outlines in the main text. 

\begin{theorem}\label{beginningtheorem}
There exists a graph with adjacency $\mathbf{A}$ such that

\resizebox{0.47\textwidth}{!}{$
\operatorname{sign}(\mathbf{A}) \neq \operatorname{sign}\left(\left[\colvector{\mathbf{z}_1} \colvector{\mathbf{z}_2} \right] \times   \left[
\begin{array}{cc}
\rowvector{\mathbf{z}_3^\top} \\
\rowvector{\mathbf{z}_4^\top} 
\end{array}
\right] \right)  = \operatorname{sign}\left(\mathbf{RC} \right)$}

for $\mathbf{R}, \mathbf{C}^\top \in \mathbb{R}^{N\times 2}$.
\end{theorem}

\begin{proof}
Appendix $3.1$. The construction is performed by assigning particular sign combinations to the entries of $\operatorname{sign}(\mathbf{A})$. In this way, a choice of latent representations\footnote{Latent vectors are denoted by $\mathbf{r}_i$ to be consistent with the notation used in the Appendix.} $\mathbf{r}_1$, $\mathbf{r}_2$ for nodes $1, 2$ satisfying the relationship pictorially depicted in Figure~\ref{threekindsofnodes} (a) can be enforced. If node $3$, represented by $\mathbf{r}_3$ (not shown), is to be connected to nodes $4,6$ and disconnected with nodes $5,7$, this implies there exists a line $\ell_3$ passing through the origin $O$ in $\mathbb{R}^2$ which leaves $\mathbf{r}_4$, $\mathbf{r}_6$ on one side of the induced hyperplane while leaving $\mathbf{r}_5$, $\mathbf{r}_7$ on the other. Clearly, this is impossible.  
\end{proof}
\begin{remark}
Instead of learning $\operatorname{sign}(\mathbf{A})$, one may consider learning $-\operatorname{sign}(\mathbf{A})$ as the latter is also a binary encoding. But even in this case, if there exists $\mathbf{R},\mathbf{C}$ such that 

\resizebox{0.47\textwidth}{!}{$
-\operatorname{sign}(\mathbf{A}) = \operatorname{sign}\left(\left[\colvector{\mathbf{z}_1} \colvector{\mathbf{z}_2} \right] \times   \left[
\begin{array}{cc}
\rowvector{\mathbf{z}_3^\top} \\
\rowvector{\mathbf{z}_4^\top} 
\end{array}
\right] \right)  = \operatorname{sign}\left(\mathbf{RC} \right)$,}

then this must imply that 

\resizebox{0.47\textwidth}{!}{$
\operatorname{sign}(\mathbf{A}) = \widetilde{\operatorname{sign}}\left(\left[\colvector{\mathbf{z}_1} \colvector{\mathbf{z}_2} \right] \times   \left[
\begin{array}{cc}
\rowvector{-\mathbf{z}_3^\top} \\
\rowvector{-\mathbf{z}_4^\top} 
\end{array}
\right] \right)  = \operatorname{sign}\left(\mathbf{R}(-\mathbf{C}) \right),$}

where $\widetilde{\operatorname{sign}}: \mathbb{R} \to \{+,-\}$ takes the negative sign on $\mathbb{R}_{<0}$. An analogous argument can be made to demonstrate the impossiblity of such a result for general $\mathbf{A}$.
\end{remark}

\begin{figure}[h!]
  \begin{subfigure}[b]{0.15\textwidth}
    \includegraphics[width=\textwidth]{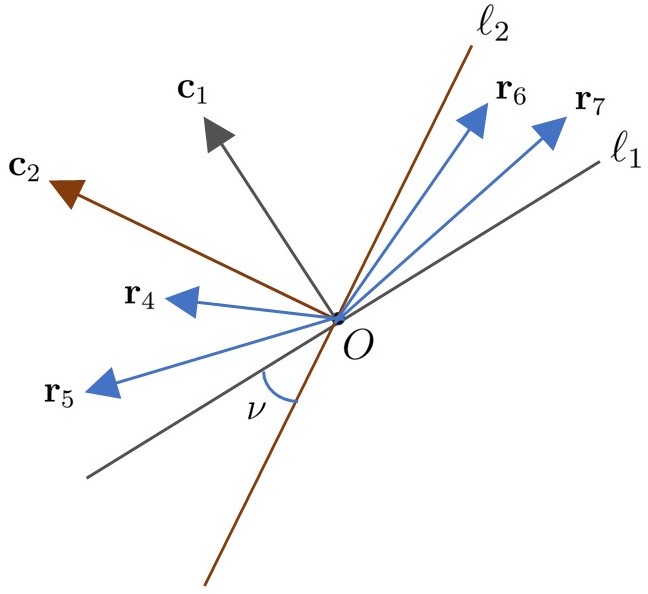}
    \caption{ }
  \end{subfigure}
    \begin{subfigure}[b]{0.15\textwidth}
    \includegraphics[width=\textwidth]{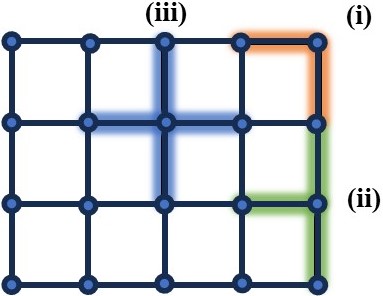}
    \caption{ }
  \end{subfigure}
  \begin{subfigure}[b]{0.15\textwidth}
    \includegraphics[width=\textwidth]{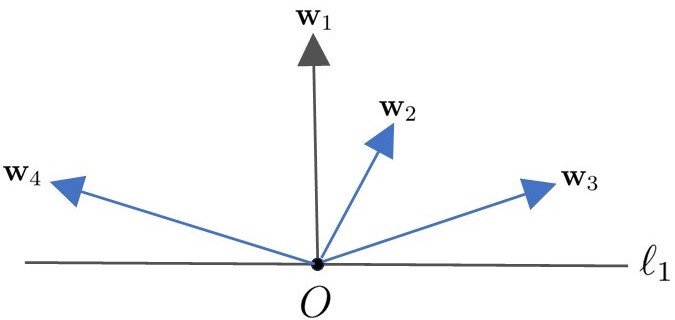}
    \caption{ }
  \end{subfigure}
  \hfill
  \caption{(a) Intermediary step in the construction of a class of graphs which do not admit a latent representation of rank $2$. (b) The three node types appearing in a planar grid graph. (c) $\mathbf{w}_1$ corresponds to the center node in (b)-\textbf{(ii)}, whereas $\mathbf{w}_i$ for $i \in [2,4]$ depicts the neighboring nodes. For each of these nodes to be disconnected, any two distinct $\mathbf{w}_i,\mathbf{w}_j$ vectors for $i,j \in [2,4]$ must form an obtuse angle. }
  \label{threekindsofnodes}
\end{figure}

An alternative proof strategy is formed by studying the permissible low-rank representations of induced subgraph structures. It is clear that the rank of a graph must be no smaller than the rank of a node-induced subgraph.
\begin{theorem}\label{gridgrapharenotrank2}
Any planar (two dimensional) square grid graph with more than $4$ nodes cannot be rank $2$.    
\end{theorem}
\begin{proof}
Appendix $3.2$. It suffices to show that there exists an induced subgraph of rank greater than $2$. We target the $3$-star graph shown in Figure~\ref{threekindsofnodes} (b)-\textbf{(ii)}, where the central node is represented by $\mathbf{w}_1$ in (c). Afterwards, the pigeonhole principle gives that any three vectors which belong in the interior of the hyperplane induced by $\ell_1$ in the direction of $\mathbf{w}_1$ must form at least one pairwise acute angle (thereby admitting a positive value under the inner product). This induces at least $1$ connection between the non-central nodes, violating a core property of star graphs that non-central nodes are pairwise disconnected. 
\end{proof}

The proof of Theorem~\ref{gridgrapharenotrank2} is particularly insightful in its use of the pigeonhole principle. Any two vectors in a space isomorphic to $\mathbb{R}^{N}$ has a natural notion of an angle induced by the inner product, and two nodes are connected if and only if their vector representations form an acute angle. Due to the restriction that the (minimum) angle between any two vectors must lie between $[0,\pi/2)$ for the corresponding nodes to be connected, we quickly consume the available degrees of freedom for latent vectors to accurately reconstruct the original graph. The ideas behind the proofs of Theorems~\ref{beginningtheorem} and~\ref{gridgrapharenotrank2} can be collected into the following theorem, which subsumes the previous results. 
\begin{figure*}[h!]
    \begin{subfigure}[b]{1\textwidth}
    \includegraphics[width=\textwidth]{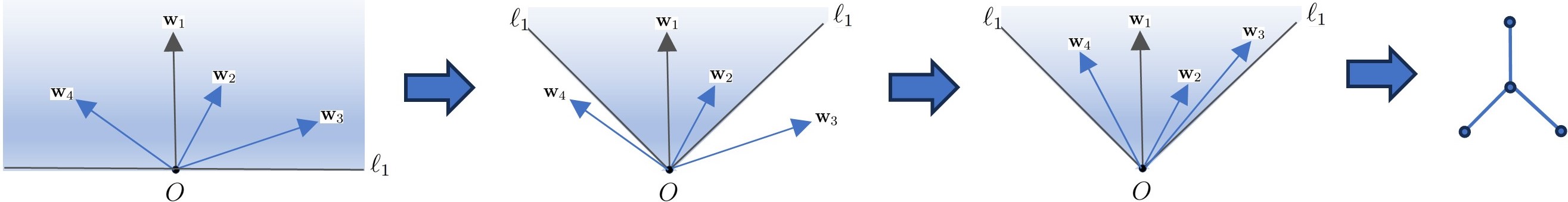}
    \caption{Utilizing cutoffs to represent the $3$-star graph in two dimensions.}
  \end{subfigure}
  \hfill
  \begin{subfigure}[b]{0.24\textwidth}
    \includegraphics[width=\textwidth]{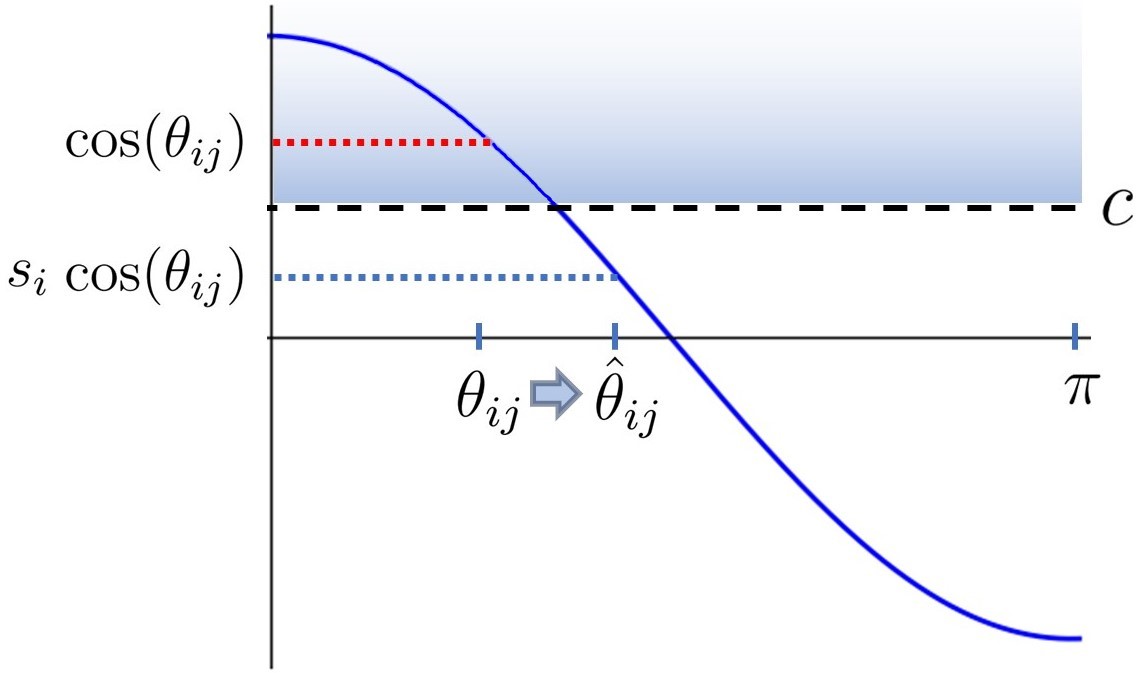}
    \caption{$0<s_i< 1$}
  \end{subfigure}
  \hfill
    \begin{subfigure}[b]{0.24\textwidth}
    \includegraphics[width=\textwidth]{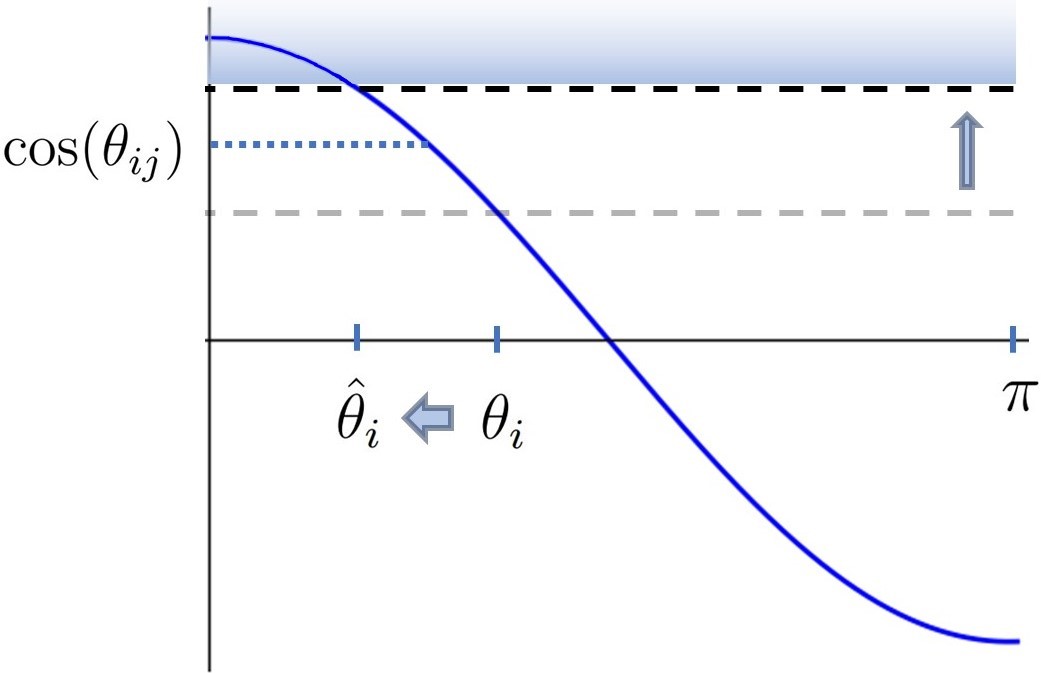}
    \caption{$0<s_i< 1$}
  \end{subfigure}
  \hfill
  \begin{subfigure}[b]{0.24\textwidth}
    \includegraphics[width=\textwidth]{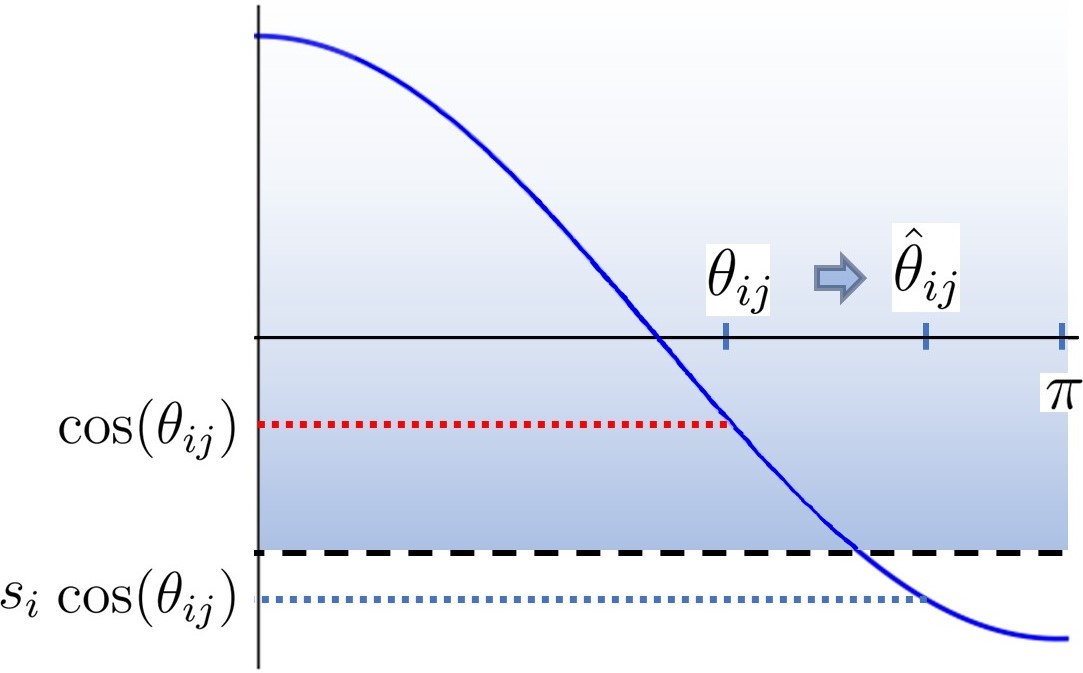}
    \caption{$1<s_i$}
  \end{subfigure}
  \hfill
  \begin{subfigure}[b]{0.24\textwidth}
    \includegraphics[width=\textwidth]{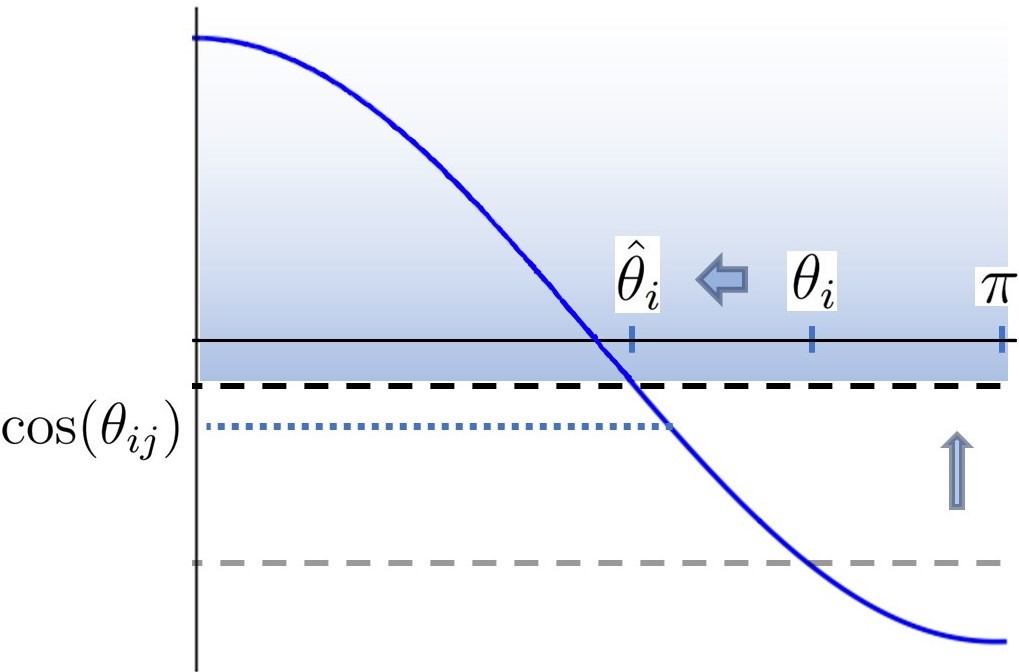}
    \caption{$1<s_i$}
  \end{subfigure}
  \hfill
  \caption{In (b-e), the connection of node $i$ with node $j$ is severed due to $s_i$ where a cutoff value $c$ is depicted as a solid dashed line. (b-c) shows the action of small $s_i$ on $\theta_{ij}$, $\theta_i$. $\theta_{ij}$ is pushed to $\pi/2$ while $\theta_i$ is pulled toward the origin. (d-e) shows a similar phenomenon, but for larger $s_i$. 
  }
  \label{gentle}
\end{figure*}
\begin{theorem}\label{realstargraph}
The rank of an $N-1$ star graph is lower bounded by $(N+1)/2$.     
\end{theorem}   
\begin{proof}
Let $\mathbf{w}_1 \in \mathbb{R}^f$ represent the central node in the latent space. As  $\mathbf{w}_1\neq \overline{0}$, set $\mathbf{w}_1 = \mathbf{e}_1$ without loss of generality for $\{\mathbf{e}_i\}_{1:f}$ the standard basis. Then, $E = \{\pm\mathbf{e}_i\}_{2:f}$ with $2f-2$ elements is the maximal set satisfying perpendicularity with $\mathbf{e}_1$. Furthermore, there exists no latent vector $\mathbf{w} \neq \mathbf{e}_1$ satisfying $\langle \mathbf{w},\mathbf{e}_1\rangle > 0$ and $\langle \mathbf{w},\mathbf{e}\rangle \le 0$ for $\forall \mathbf{e} \in E$. This implies that the $(2f-2)$-star graph cannot be represented in $\mathbb{R}^f$, and rearranging $N-1 = 2f-2$ gives the bound. 
\end{proof}
In contrast, Theorem~\ref{complexstargraph} shows that the complex field is far more economical for embedding star graphs.

\begin{theorem}\label{complexstargraph}
An $(N-1)$-star graph is complex rank $1$ for any $N$.    
\end{theorem}   
A proof is given in Appendix $3.3$ by construction. Theorem~\ref{complexstargraph} confirms that moving to the complex latent space allows for alternative graph structures to be faithfully (i.e. $||\mathbf{A}-\hat{\mathbf{A}}||_F = 0$ for $\hat{\mathbf{A}}$ the computed adjacency) represented by the low-dimensional encoding. Finally, we note that Theorem~\ref{realstargraph} implies a lower bound on the minimal number of latent dimensions required to reconstruct an arbitrary graph.
\begin{corollary}
For graph $G$, let $H$ be the largest induced $(N_H-1)$-star subgraph. Then any faithful real latent encoding of $G$ must possess at least $\lceil(N_H+1)/2\rceil$ dimensions. 
\end{corollary}

\section{Decoder Augmentation with Cutoffs}\label{cutoff}
In Section~\ref{lowerboundinggraphrank}, we saw that the inner product decoder imputes a connection whenever an acute angle is formed between the latent vectors, quickly depleting the available degrees of freedom. To remedy this issue, we motivate a strategy using \textit{cutoffs} to relax the angle constraint $[0,\pi/2)$ that indicates a connection between nodes, which assigns a different angle connection range $[0,\theta_i)$ to every node $i$. For simplicity, we work in the real space. To enhance clarity, we begin with a visual explanation.

In Figure \ref{gentle} (a), the leftmost picture depicts the region of connectivity for node $1$, where a connection to all nodes with latent vector representations within the region is established. In the vanilla decoder $\operatorname{sign}(\mathbf{ZZ}^\top)$, the region of connectivity is a hyperplane induced by the line passing through the origin which is perpendicular to the latent vector. This indicates that nodes $2,3$ and $2,4$ are also connected as the angle between their latent embeddings is acute. Introducing a cutoff then morphs the connective region from a hyperplane to a cone due to the action of $s_i$ (see following discussion) on the connective angle $\theta_1$, severing the connection with node $1$ and nodes $3,4$. In the third picture, the GNN learns to modify the latent embeddings and cutoffs so that the connection to node $1$ is reestablished, as well as sharpening the pairwise connective angles between $\mathbf{w}_i$,$\mathbf{w}_j$ for $i,j \in [2,4]$ causing nodes $2,3,4$ remain disconnected. In the rightmost picture, the latent embedding is decoded to represent the $3$-star graph, which is not expressible in two dimensions without cutoffs. The following discussion elucidates this procedure in finer detail.

Consider an $f$-dimensional latent embedding

\resizebox{0.47\textwidth}{!}{$
\operatorname{sign}(\mathbf{A}) =  \operatorname{sign}\left(\left[
\begin{array}{cc}
\rowvector{\mathbf{w}_1^\top} \\
\rowvector{\mathbf{w}_2^\top} \\ 
\vdots \\ 
\rowvector{\mathbf{w}_N^\top}
\end{array}
\right] \times \left[\colvector{\mathbf{w}_1} \colvector{\mathbf{w}_2} \dots \colvector{\mathbf{w}_N} \right] \right)
$}

where $\mathbf{w}_i$ are the rows of the latent matrix $\mathbf{Z}$. Multiplying each $\mathbf{w}_i$ by a positive scalar does not impact the sign of the inner product $\langle \mathbf{w}_i,\mathbf{w}_j\rangle = ||\mathbf{w}_i||_{\ell_2}||\mathbf{w}_j||_{\ell_2} \cos(\theta_{ij})$, thus we normalize $\mathbf{w}_i$ to unit $\ell_2$-length. In what follows, we will multiply $\mathbf{w}_i$ by a small positive scalar $s_i$, but due to constraining $||\mathbf{w}_i||_{\ell_2} = 1$, the action of the scalar will be interpreted to have been passed onto the angle $\theta_{ij}$, $\hat{\theta}_{ij} = s_i(\theta_{ij})$. In other words, $\hat{\theta}_{ij}$ will be computed from $\langle s_i\mathbf{w}_i,\mathbf{w}_j\rangle = \cos(\hat{\theta}_{ij})$. For clarity, we leave all other $\mathbf{w}_j$ untouched for $j \neq i$ and work singularly with $\mathbf{w}_i$. 

Under this interpretation, the action of $s_{i}$ with small magnitude pushes $\theta_{ij} \in [0,\pi/2) \bigcup (\pi/2,\pi]$ toward $\hat{\theta}_{ij}\approx \pi/2$, while large $s_i$ pulls $\theta_{ij}$ toward $\hat{\theta}_{ij}\approx0,\pi$ depending on whether $\theta_{ij}$ was on the left or right of $\pi/2$, respectively (Figure \ref{gentle} (b),(d)). A cutoff $c$ enforces that nodes $i,j$ are connected if and only if $\langle \mathbf{w}_i,\mathbf{w}_j\rangle > c$ (see equation~\eqref{manualcutoffdecoder}). This can be interpreted as reducing the angle of connectivity assigned to node $i$ from $\theta_i = \pi/2$ to $0< \theta_i < \pi/2$ (Figure \ref{gentle} (a)). The action of $s_i$ with small magnitude then pulls $\theta_i$ closer to the origin to form $\hat{\theta}_i \approx 0$ (Figure \ref{gentle} (c)), while large $s_i$ pushes $\theta_i$ away to $\hat{\theta}_i \approx \pi/2$ (Figure \ref{gentle} (e)).

Now, acting on all other $\mathbf{w}_j$ simultaneously with the $s_j$'s motivates the decoder
\begin{equation}\label{manualcutoffdecoder}
\operatorname{sign}(\mathbf{A}) = \operatorname{sign}(\mathbf{ZZ}^\top - c\mathbf{1}\mathbf{1}^\top),    
\end{equation}
where $\mathbf{1} \in \mathbb{R}^{N}$ is the standard $1$-vector. If the $\mathbf{w}_i$ are normalized to unit length, this decoder assigns a singular angle of connectivity $\theta_i$ for $\mathbf{w}_i$. But as this normalization need not hold in practice, the value of $\theta_i$ \textit{varies} depending on which $j$ is being considered for each $\mathbf{w}_i,\mathbf{w}_j$ pair\footnote{One way to see this is to assume $||\mathbf{w}_1||_{\ell_2} = 1$, $||\mathbf{w}_j||_{\ell_2} \neq 1$ for $j \neq 1$. The lengths of the $\mathbf{w}_j$ act on $\theta_1$ accordingly.}. Enriching the degree of freedom in which the decoder is allowed to vary the connective angle $\theta_i$ for each $\mathbf{w}_i,\mathbf{w}_j$ pair, we have the decoding architecture 
\begin{equation}\label{notfinaldecoder}
\operatorname{sign}(\mathbf{A}) = \operatorname{sign}(\mathbf{Z}_1\mathbf{Z}_1^\top - \mathbf{Z}_2\mathbf{CZ}_2^\top),    
\end{equation}
where $\mathbf{Z}_i \in \mathbb{R}^{N\times f_i}$ and $\mathbf{C}\in \mathbb{R}^{f_2\times f_2}$ is diagonal with $C_{ii}=c_i$.

\subsection{Complex latent embeddings}
We remark that decoder~\eqref{notfinaldecoder} is reminiscent of decoding in the complex field, in that if $\mathbf{Z} = \mathbf{Z}_r + i \mathbf{Z}_i,$ then $\mathfrak{Re}(\mathbf{ZZ}^\top) = \mathbf{Z}_r\mathbf{Z}_r^\top - \mathbf{Z}_i\mathbf{Z}_i^\top$. Thus, the decoding scheme derived here can also be seen as a generalization of a decoder acting on complex latent representations. The imaginary component of a complex embedding is responsible for discovering the optimal cutoff tolerance for each $i,j$ node pair, which is learnable as a parameter.

As a practical matter, working in the complex field is natural during the decoding phase. Any symmetric real adjacency matrix has an orthogonal eigendecomposition $\mathbf{A} = \mathbf{U}\boldsymbol{\Lambda}\mathbf{U}^\top$, and there may exist no real latent embedding such that $\mathbf{A} \approx \mathbf{ZZ}^\top$ in contrast to complex latent embeddings which are immediate from the eigendecomposition.

\begin{figure*}[h!]
    \begin{subfigure}[b]{1\textwidth}
    \includegraphics[width=\textwidth]{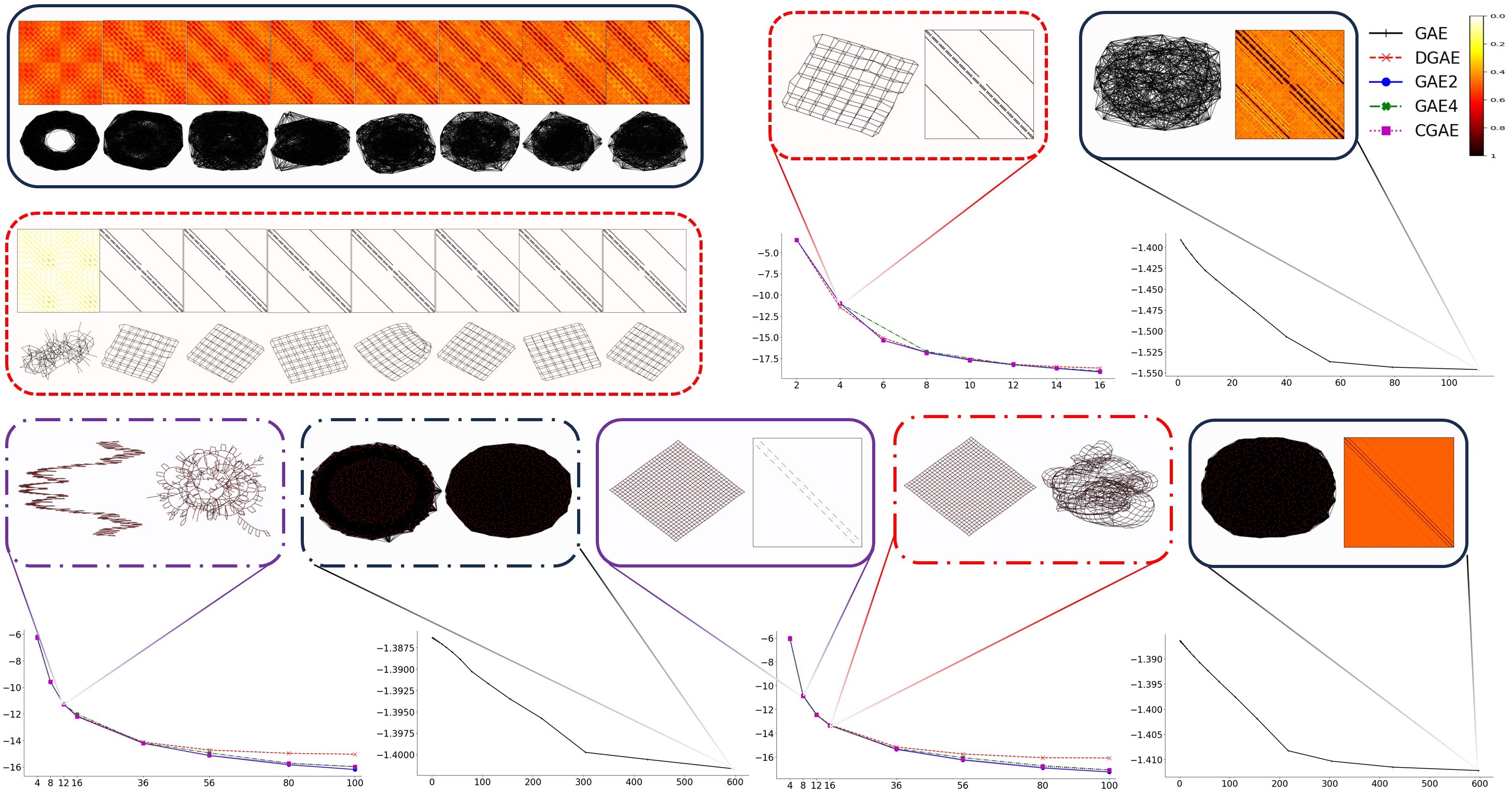}
  \end{subfigure}
  \hfill
  \caption{Introducing cutoffs offer drastic improvements to the representation capacity of the inner product (GAE) architecture. All numerical plots display the log-normalized distance against $h_2$, and dotted lines indicate a probabilistic decoding whereas solid lines denote a sign decoding. Dotted-dash lines indicate both a sign (left) and probabilistic (right) reconstruction. $h_2$ is varied $2$ to $16$, equispaced, in upper left where GAE (top) and DGAE (bottom) are compared. Training is done for $30000$ epochs on NVIDIA GeForce RTX 3060 GPU with learning rate $10^{-4}$ and weak regularization $\lambda = 10^{-7}$. Boundary color coding is matched with legend architecture colors given in the upper right, and all colormaps visualize computed adjacencies constrained by latent rank. }
  \label{experiment}
\end{figure*}

\subsection{Generalizations}\label{generalizations}

We can enhance the expressivity of the decoder further by noting that~\eqref{notfinaldecoder} is subsumed by
\begin{equation}\label{diagdecoder}
\operatorname{sign}(\mathbf{A}) = \operatorname{sign}(\mathbf{Z}_1\mathbf{C}_1\mathbf{Z}_1^\top - \mathbf{Z}_2\mathbf{C}_2\mathbf{Z}_2^\top), 
\end{equation}
where the weights of the decoder, $\mathbf{C}_i$, are diagonal. The RHS is equivalent to $\operatorname{sign}(\mathbf{Z} \mathbf{C} \mathbf{Z}^\top)$ for $\mathbf{Z} = [\mathbf{Z}_1,\mathbf{Z}_2]$ and $\mathbf{C} = \operatorname{diag}(\mathbf{C}_1,-\mathbf{C}_2)$. This formulation allows the algorithm to learn a latent embedding and cutoff dimensions $f_1, f_2$ given their sum $f_1+ f_2$. Note that this architecture is very similar in structure to the eigendecomposition. Furthermore, the matrix rank of $\mathbf{Z} \mathbf{C} \mathbf{Z}^\top$ is upper bounded by $f_1 + f_2 \ll N$.

Alternatively, consider the decoder
\begin{equation}\label{mdecoder}
\operatorname{sign}\left(\sum_{n = 0}^k (-1)^n \mathbf{C}_{4n}\mathbf{Z}_n\mathbf{C}_{4n+1}\mathbf{C}_{4n+2}\mathbf{Z}_n^\top \mathbf{C}_{4n+3}\right)
\end{equation}
for diagonal, (upper or lower) bidiagonal, or tridiagonal $\mathbf{C}_i$. The intuition is to act on the latent encoding $\mathbf{Z}_j$ from the left and the right by highly sparse matrix multiplication $\mathbf{C}_i$ to allow $\mathbf{Z}_j$ to uphold maximal flexiblity while maintaining economical cost. This decoder can be interpreted as accumulating the messages from $k+1$ linear layers with weights $\mathbf{C}_i$ acting as a `pivot' to optimize the latent representations $\mathbf{C}_i \mathbf{Z}_j \mathbf{C}_k$. 

Activating $\mathbf{C}_{4n+1}$ or $\mathbf{C}_{4n+2}$ increases the number of learnable parameters by $\mathcal{O}(f_i)$, whereas activating $\mathbf{C}_{4n}$ or $\mathbf{C}_{4n+3}$ costs $\mathcal{O}(N)$. In particular, activating all entries of the former two $\mathbf{C}_i$ leaves minimal impact on the number of parameters while enhancing the expressive capacity of the neural net. Thus while making only modest changes to the general architecture, these decoding strategies should drastically enhance the representation capacity by utilizing cutoffs, which are straightforward to implement.

\section{Architectures and Experiments}
\subsection{Architectures}

Arguably, the most well-known GCN utilizing the inner product decoder is the seminal \citeauthor{VGAE} framework. We target their autoencoding inference model (\textbf{GAE}) by learning a latent embedding for grids and chains. Previous discussions intuit that the grid and chain backbones will suffer from superfluous and erroneous node connections in the reconstruction. Introducing learnable cutoffs as in Section~\ref{generalizations} should then sever these edges during model training.

The GAE is given by
$$
\mathbf{A} \operatorname{ReLU}\left(\mathbf{A} \mathbf{X} \mathbf{W}_0\right) \mathbf{W}_1 = \mathbf{Z}_0, \quad \hat{\mathbf{A}} = \mathcal{B}\left(\sigma\left(\mathbf{Z}_0\mathbf{Z}_0^\top \right)\right),
$$
where $\mathcal{B}$ indicates elementwise Bernoulli sampling. We allow a flexible cutoff space to the decoder by introducing a sparse diagonal parameter $\mathbf{C}_0$,  
$$
\mathbf{A} \operatorname{ReLU}\left(\mathbf{A} \mathbf{X} \mathbf{W}_0\right) \mathbf{W}_1 = \mathbf{Z}_0, \quad \hat{\mathbf{A}} = \operatorname{sign}\left(\mathbf{Z}_0\mathbf{C}_0\mathbf{Z}_0^\top \right),
$$
denoted by \textbf{DGAE}~\eqref{diagdecoder}. $\mathbf{C}_0$ is learned by propagating the latent representation $\mathbf{Z}_0$ through a linear layer prior to decoding. 

An alternative architecture is formed via equation \eqref{mdecoder}, denoted by \textbf{$m$GAE}:
\begin{align*}
&\mathbf{A} \operatorname{ReLU}\left(\mathbf{A} \mathbf{X} \mathbf{W}_{2n+1}\right) \mathbf{W}_{2n+2} = \mathbf{Z}_n, \\
\hat{\mathbf{A}} = \operatorname{sign}&\left(\sum_{n = 0}^k (-1)^n \mathbf{C}_{4n}\mathbf{Z}_n\mathbf{C}_{4n+1}\mathbf{C}_{4n+2}\mathbf{Z}_n^\top \mathbf{C}_{4n+3}\right).
\end{align*}
We take $m$ to be the number of $\mathbf{Z}_n$. Finally, the \textbf{CGAE} is given as the complex counterpart to the GAE, where the entries of $\mathbf{A}, \mathbf{X},$ and $\mathbf{W}_i$ live in the complex field. The imaginary portion of $\mathbf{Z}_0\mathbf{Z}_0^\top$ is truncated prior to applying the $\operatorname{sign}$ map to form $\hat{\mathbf{A}}$, and $\operatorname{ReLu}$ acts independently on each dimension during training, i.e. $\operatorname{ReLu}(a+ib) = \operatorname{ReLu}(a) + i\operatorname{ReLu}(b)$. In all nets, edges can be decoded probabilistically by replacing $\operatorname{sign}$  with $\mathcal{B}\circ \sigma(\cdot)$ and vice versa.

\subsection{Experiments}

Feature information is first propagated to $h_1$ hidden dimensions, then to $h_2$ latent dimensions during encoding. For the CGAE, we first project to $h_1/2$ \textit{complex} hidden dimensions (corresponding to $h_1$ real dimensions considering the real and imaginary parts separately), then to $h_2/2$ complex latent dimensions for fair comparison. The training loss is given by
\begin{equation*}
L\left(\mathbf{A},\mathbf{Z}\right) = \operatorname{BCE}\left(\mathbf{A},\sigma \circ \mathfrak{Re}(\mathbf{Z}\mathbf{Z}^\top ) \right) +  \lambda ||\mathbf{A} - \mathfrak{Re}(\mathbf{Z}\mathbf{Z}^\top )||_{F}.
\end{equation*}

We instantiate the $2$GAE by inactivating all $\mathbf{C}_i = \mathbf{I}$ to establish a comparison to the CGAE, and the $4$GAE by inactivating all but one diagonal $\mathbf{C}_{4n+1}$ for every $\mathbf{Z}_n$, $\mathbf{I}=\mathbf{C}_{4n}=\mathbf{C}_{4n+2}=\mathbf{C}_{4n+3}$. 

For visual clarity, we first present several results using pedagogical graph structures. Figure~\ref{experiment} targets the $8\times 8 \times 2$ (short grid), $27 \times 27$ (long grid) planar grid graphs, and a chain composed of linking $6$-cycles $140$ times (long chain). For short graphs, the hidden dimensions used were $h_1 = 30$ while varying $h_2$ from $2$ to $16$; long graphs used $h_1 = 120$ and $h_2$ was varied from $4$ to $100$. For the GAE only, $h_1$ was set to node count when plotting log-normal distance in order to meaningfully project to large latent dimension $h_2$. Appendix $4$ contains full details as well as additional results (e.g. long diverse chain of $1200$ nodes successfully embedded in only $8$ latent dimensions). Treatment of Cora and CiteSeer graphs (\citealt{CoraCiteSeer}) with 2708, 3327 nodes, respectively, are given in Appendix~\ref{realworldgraphnetworks} and similarly reinforce the strength of the augmented architectures. 

\begin{figure}[h]
\vspace{.1in}
\includegraphics[width=0.46\textwidth]{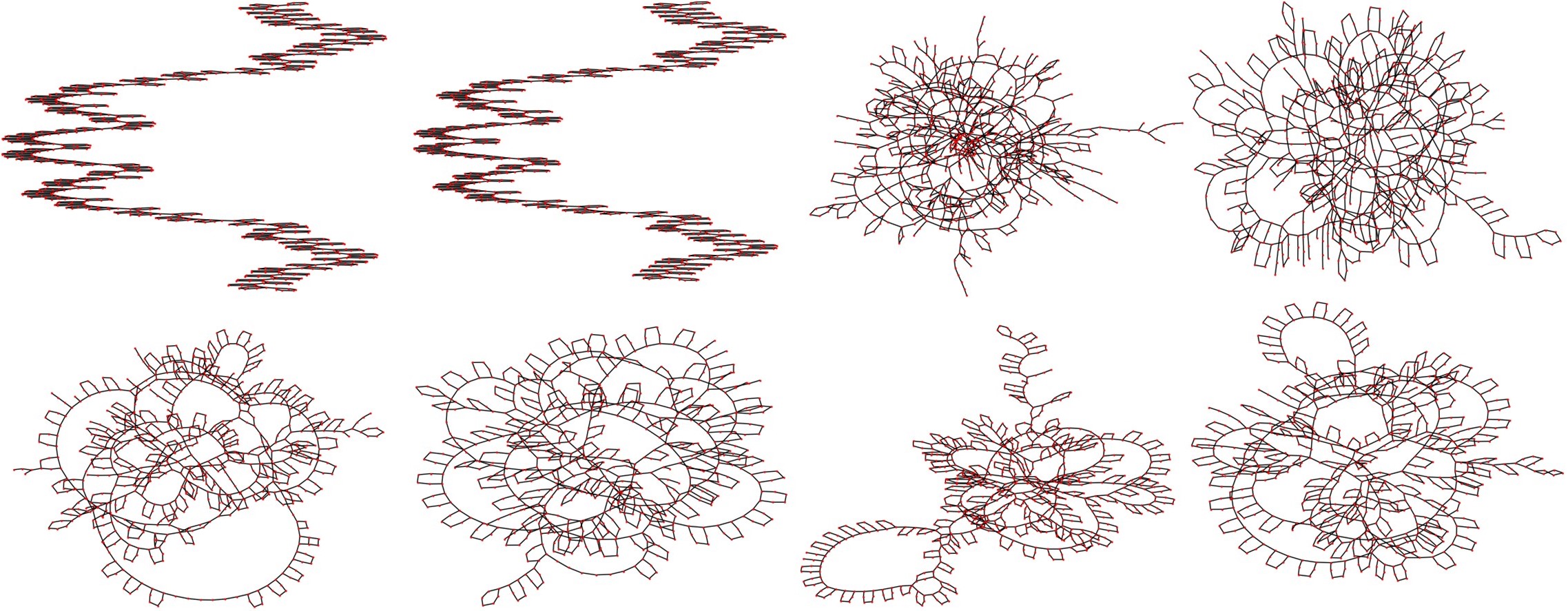}
\vspace{.1in}
\caption{(Top) Input long chain graph of $140$ linked $6$-cycles, $h_2 = 8$ deterministic reconstruction, and $h_2 = 8, 16$ probabilistic reconstructions. (Bottom) $h_2 = 36, 56, 80, 100$ probabilistic reconstructions, all synthesized via $4$GAE under identical hyperparameters as Figure~\ref{experiment}, left to right.}
\label{longchainprobabilisticreconstruction}
\end{figure}

While being prone to bias, an inherent advantage of the $\operatorname{sign}$ decoder is its ability to deterministically extract a robust structure from a low-confidence environment. All architectures augmented by cutoffs flawlessly reconstructed input graphs for $h_2 \ge 8$ (Figure $7$, Appendix $4$). We further note that a faithful reconstruction via sign decoding necessarily implies a faithful probabilistic reconstruction via $\mathcal{B} \circ \sigma$ under weaker normalization $\lambda$, in that scaling the magnitude of the latent vector entries by a large positive constant will drive the positive and negative entries of $\tilde{\mathbf{A}}$ toward $1,-1$ respectively under the sigmoid. 

The normal distance $d(\mathbf{A}, \tilde{\mathbf{A}}) = ||\mathbf{A}- \sigma(\tilde{\mathbf{A}})||_{F}^2/N^2$ confirms that the entangled pictorial reconstruction of the probabilistically sampled $\hat{\mathbf{A}}$ is the result of unrobustness in graph visualization for large node count. In all examples, the augmented architectures captured essential structural attributes such as $n$-cycles and chained structures, while the GAE struggles significantly even for very high latent dimensions. We note that the augmented architectures empirically remained stable under reduction in $h_1$ (Appendix $4.1$). 

\subsubsection{Ranks of real-world graphs}
All molecules with more than 27, 36, 40, 60 nodes from QM-9, Zinc, TU-Enzymes, TU-Protein datasets (\citealt{QM9}; \citealt{Zinc}; \citealt{TUDataset}) are treated via GAE and DGAE, totaling 35, 118, 162, 176 graphs. For all datasets, the DGAE succeeds in distilling a faithful embedding under $h_2 = 10$ latent dimensions while the GAE mostly fails for up to $h_2 = 80$. Results are summarized in Figure~\ref{rankfigure}, where the histograms are plotted as `densities' for better viewability. We note that there are datasets in which the GAE performs competitively (e.g. QM7-b), but nevertheless still fails to exceed the performance of DGAE. A shallow embedding architecture was enforced by taking $h_1 = 2h_2$ while varying $h_2$ in order to achieve maximal model compression. 

\begin{figure}[h]
\vspace{.1in}
\includegraphics[width=0.46\textwidth]{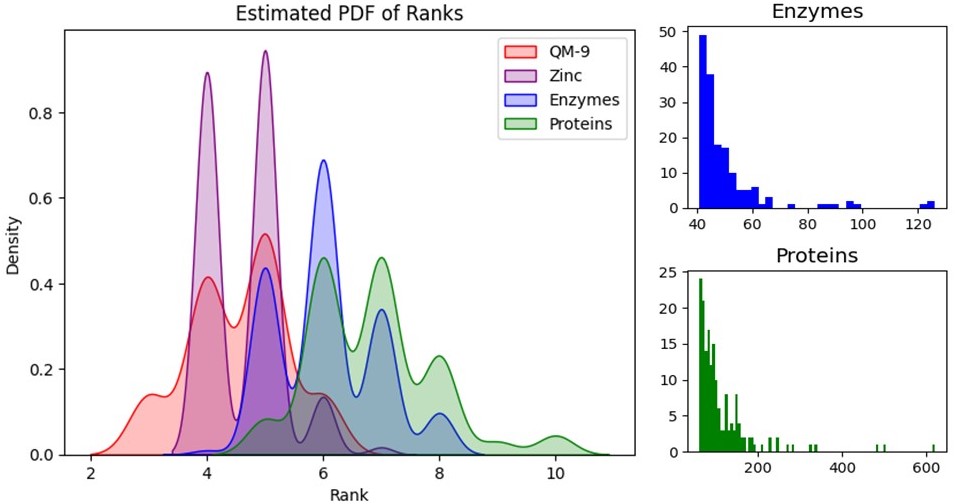}
\vspace{-.1in}
\caption{(Left) All graphs in the molecular benchmarks were faithfully embedded in $\le 10$ latent dimensions. Histograms are normalized to densities for better viewability. (Right) Node number versus molecule count histograms of filtered Enzyme and Protein datasets. Note that all molecules in filtered QM9 dataset have 29 nodes, and 37 nodes for all but one (38 nodes) for the Zinc dataset [both not shown].}
\label{rankfigure}
\end{figure}

\section{Conclusion}

In this paper, we show that the widely used inner product decoder is unable to efficaciously establish node relational strength for graph memorization tasks. Furthermore, algebraic arguments advocate utilitarian design principles such as latent complexification that provably enhances the performance of GNN architectures utilizing the inner product decoder. To our knowledge, this is the first theoretical study elucidating the pervasive phenomenon of the often inadequate performance of inner product decoders for graph structured data. 

\section{Acknowledgments}
We would like to thank the anonymous reviewers for their feedback. Su Hyeong Lee is supported by the Kwanjeong Educational Foundation Scholarship and the University of Chicago McCormick Fellowship. 




\bibliography{mybib.bib}
\bibliographystyle{icml2024}

\newpage
\appendix
\onecolumn
\begin{figure}[h]
\vspace{.1in}
\includegraphics[width=\textwidth]{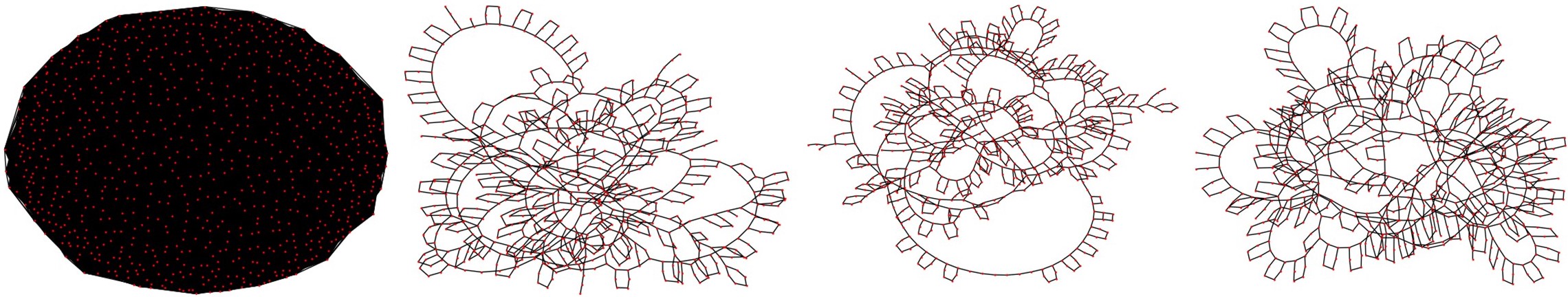}
\vspace{.1in}
\caption{VAE with $h_2 = 597$, and DVAE, $4$VAE, CVAE with $h_2 = 36$, all reconstructed from latent projections of the long chain of $140$ linked $6$-cycles via $\mathcal{B}\circ \sigma$, left to right. We prove that inner product decoding, while universally utilized in machine learning, is not suitable for establishing node strength relations in graph reconstruction problems due to restrictive degrees of freedom in the latent space. Introducing cutoffs, a straightforward yet drastically powerful design choice, significantly augments the expressive capacity of the GCN architecture without deviating from the inner product framework. Identical hyperparameters as Figure $6$ in the main paper were used.}
\label{suitablepicture}
\end{figure}

\section{Kipf \& Welling's GAE Framework}

The adjacency $\mathbf{A} \in \mathbb{R}^{N\times N}$ of an unweighted, undirected graph is a Hermitian matrix encoding binary information, where the entry $1$ indicates a connection between two nodes while $0$ indicates a disconnection. Given a feature matrix $\mathbf{X} \in \mathbb{R}^{N\times d}$, a \textit{Graph Convolutional Network (GCN)} can be utilized to find a latent mapping $\mathbf{X}\mapsto \mathbf{Z} \in \mathbb{R}^{N\times f}$, where $f$ is the latent dimension. In Kipf and Welling (2016), a simple variational inference model is given:
$$
q(\mathbf{Z} \mid \mathbf{X}, \mathbf{A})=\prod_{i=1}^N q\left(\mathbf{z}_i \mid \mathbf{X}, \mathbf{A}\right), \quad \text{ for }  \quad q\left(\mathbf{z}_i \left|   \mathbf{X}, \mathbf{A}\right.\right)=\mathcal{N}\left(\mathbf{z}_i \left| \boldsymbol{\mu}_i, \operatorname{diag}\left(\boldsymbol{\sigma}_i^2\right)\right.\right) .
$$
The mean matrix $\boldsymbol{\mu}$ with columns  $\boldsymbol{\mu}_i$ and the log-variance matrix $\log\left(\boldsymbol{\sigma}\right)$ are found by message passing through the two layer $\operatorname{GCN}(\mathbf{X}, \mathbf{A})=\tilde{\mathbf{A}} \operatorname{ReLU}\left(\tilde{\mathbf{A}} \mathbf{X} \mathbf{W}_0\right) \mathbf{W}_1$. All weights are real while the first layer parameters $\mathbf{W}_0$ are tied to be identical. $\tilde{\mathbf{A}}=\mathbf{D}^{-\frac{1}{2}} \mathbf{A} \mathbf{D}^{-\frac{1}{2}}$ was taken to be the symmetrically normalized adjacency matrix to control for node degree variance. After computing the latent embedding $\mathbf{Z}$, a generative decoding model is formed by taking inner products, 
$$
p(\mathbf{A} \mid \mathbf{Z})=\prod_{i=1}^N \prod_{j=1}^N p\left(A_{i j} \left| \mathbf{z}_i, \mathbf{z}_j\right.\right), \quad\text { for } \quad p\left(A_{i j}=1 \left| \mathbf{z}_i, \mathbf{z}_j\right.\right)=\sigma\left(\mathbf{z}_i^{\top} \mathbf{z}_j\right),
$$
where $\sigma$ denotes the sigmoid. In essence, the decoding is done by interpreting entries of $\sigma\left(\mathbf{ZZ}^\top\right)$ as Bernoulli probabilities of successful connections between the nodes, which forms a distribution over (not necessarily symmetric) adjacency matrices $\hat{\mathbf{A}}$. $\hat{\mathbf{A}}$ may be artificially symmetrized by substituting $\hat{A}_{ij} \leftarrow 1$ if the $ij$-th entry of $\left(\hat{\mathbf{A}} + \hat{\mathbf{A}}^\top \right)/2$ is nonzero, which permits an interpretation as an undirected graph. 

This model is minimized over the \textit{Evidence Lower Bound (ELBO)}, which induces an adversarial competition between a reconstruction of the desired adjacency $\mathbf{A}$ and the KL-divergence, meant to constrain the generative model to remain faithful to its Gaussian prior. During this process, the latent dimension $f$ is tuned empirically as a hyperparameter until the preferred model behavior is observed. Especially as the graph grows high dimensional, a significantly lower value of $f$ which reproduces the original adjacency $\mathbf{A}$ is desirable due to the prohibitive computational cost. Note that the removal of the variance matrix $\log(\boldsymbol{\sigma})$ deduces an autoencoder model, which is commonly denoted by GAE within the literature. 

This paper formalizes the notion of dimensionality in latent representations using the \textit{sign rank}, and supplies examples of pedagogical graph structures (stars, grids, and chains) for which complex decoding structures permit significantly lower-dimensional latent encoding to be used. In particular, we provide a theoretical justification as to why transitioning to the complex field for Graph Neural Networks (GNNs) greatly diversifies the range of permissible low-dimensional latent embeddings while minimally sacrificing expressivity. Guided by theory, we design a decoding architecture which expands the representation capacity of low dimensional embeddings and subsumes the expressivity of complex GNNs. 

\section{Generation of Rank $2$ and Rank $3$ Graphs}

The easiest approach to generating rank $2$ graphs that we could find was to take conjugate multiples of a real matrix function with two column dimensions. For $k_1, \dots, k_n \in \mathbb{R}_{\ge 0}$, let
\begin{equation}\tag{\ref{choosethismatfunc}}
\mathbf{Z} = \left(\begin{array}{cc}
 \cos(k_1  x)    &  \sin(k_1  x)\\
 \cos(k_2  x)     &  \sin(k_2  x) \\
\vdots   & \vdots   \\
\cos(k_N  x)  &  \sin(k_N  x) \\
\end{array} \right), \quad \boldsymbol{\tilde{A}} = \mathbf{ZZ}^\top.    
\end{equation}
The entries of the low rank representation $\tilde{\mathbf{A}}$ are expressible as a single trigonometric function, $\tilde{A}_{ij} = \cos ((k_i - k_j) x)$. Out of all matrix functions available, the reason for choosing~\eqref{choosethismatfunc} is given by Lemmas \ref{easylemma}, \ref{easylemma2}, which naturally intuits that the expression should be capable of producing all sign combinations of rank $2$.

\subsection{Proofs of Lemmas~\ref{easylemma} and~\ref{easylemma2}}
\begin{lemma}\label{easylemma3}
Let $k_i : = \sqrt{p_i}$ where $p_1<p_2<\dots$ are any sequence of positive integer primes. Limiting the indices to the lower triangular portion $i>j$, the periods of $\tilde{A}_{ij}$ can never match. That is, there exists no $n_1,n_2 \in \mathbb{Z}_{\neq 0}$ such that $n_1 t_1 = n_2 t_2$, where $t_1,t_2$ are periods of $ \tilde{A}_{ij}, \tilde{A}_{i^\prime,j^\prime}$ for $\{i,j\} \neq \{i^\prime,j^\prime \}.$
\end{lemma}

\begin{proof}
The period $t_{ij}$ of $\tilde{A}_{ij}$ is given by $2\pi/(|k_i - k_j|)$. Let us assume for contradiction that $n_1 t_1 = n_2 t_2$ is satisfied. Then, we have
\begin{equation*}
    \frac{2n_1\pi}{|k_i - k_j|} = \frac{2n_2\pi}{|k_{i^\prime} - k_{j^\prime}|} \implies n_1^2 (k_{i^\prime} - k_{j^\prime})^2 = n_2^2 (k_i - k_j)^2.
\end{equation*}
Thus, there must exist an integer $m$ such that 
$$
m = 2n_1^2 \sqrt{p_{i^\prime}p_{j^\prime}} - 2n_2^2 \sqrt{p_{i}p_{j}},
$$
but taking squares give that $\sqrt{p_{i^\prime}p_{j^\prime}p_{i}p_{j}}$ must be rational, a contradiction to $\{i,j\} \neq \{i^\prime,j^\prime \}$.
\end{proof}

Therefore, each entry of $\tilde{\mathbf{A}}$ is of the form $\cos(\theta x)$, where the $\theta$ are unique to every lower triangular entry. Note that $x = 0$ initializes all such entries to be $1$, and Lemma \ref{easylemma} guarantees that this will never happen again as we vary $x$ over the real line. However, the values of any two non-diagonal entries can become arbitrarily close near $1$ as we vary $x$ far afield. This is formally expressed by Lemma~\ref{easylemma2}, which possesses an elegant proof. 

\begin{lemma}
Let $t_1$, $t_2$ be as in Lemma \ref{easylemma3}. Then for any $0<|\varepsilon| < 1$ there exists $n_1,n_2 \in \mathbb{Z}_{>0}$ such that $|n_1 t_1 - n_2 t_2| < \varepsilon.$
\end{lemma}
\begin{proof}
We assume without loss of generality that $t_1 < t_2$. Then, there exists two sequences of increasing integers $m_1 \ll m_2 \ll \dots$ and $\ell_1 \ll \ell_2 \ll \dots$ which satisfies $n_2t_2 - n_1t_1 \in [0,t_1]$, where $n_1 = \ell_i$, $n_2 = m_i$ for $i = 1, 2, \dots$. We divide the interval $[0,t_1]$ into $k = \operatorname{ceil}(t_1/|\varepsilon|) + 1$ subintervals. By the pigeonhole principle, we know that at least two values of $n_2t_2 - n_1t_1$ must fall within the same subinterval by considering $n_1 = \ell_i$, $n_2 = m_i$ for $i = 1, 2, \dots, k+1$. Denote these two indices $(\ell_1,m_1)$ and $(\ell_2,m_2)$ for simplicity. This immediately gives that $|(m_2-m_1)t_2 - (\ell_2-\ell_1)t_1| < |\varepsilon|$, which concludes the proof.
\end{proof}

Taken together, Lemmas~\ref{easylemma} and~\ref{easylemma2} ensure that any two distinct lower triangular entries of $\mathbf{\tilde{A}}$ may never simultaneously take the value $1$ for any $x \neq 0$, but may become arbitrarily close to $1$ for large enough $x$. At this point, their relationship resets from the perspective of their periods. As we demand that $|\varepsilon|$ converges to $0^+$, it is expected that this reset happens farther afar in the real line. Due to the pairwise periodic regularity of the entries of $\mathbf{\tilde{A}}$ which emerges despite their perpetually imbalanced periods (Lemma~\ref{easylemma}), it is clear that all possible sign combinations of rank $2$ should be taken as we vary $x \in \mathbb{R}$.

The following Proposition formally proves that this is indeed the case. The proof is given by a straightforward geometric argument on $S^1$, which is detailed in the main paper.  
\begin{proposition}
Let $\mathbf{Z}$, $\tilde{\mathbf{A}}$ be as in~\eqref{choosethismatfunc}. Then, for any rank $2$ adjacency $\mathbf{A}$ representing a connected graph, there exists $\left(k_1, \dots, k_n\right)$ such that $\operatorname{sign}(\mathbf{A}) = \operatorname{sign}(\boldsymbol{\tilde{A}})$.
\end{proposition}

It is clear that an analogous geometric argument on $S^2$ generalizes Proposition~\ref{firstproposition} to rank $3$ graphs. The only difference is that the argument is made using spherical coordinates instead of polar coordinates. 
\begin{proposition}
For $k_1, k_1^\prime, \dots, k_n,k_n^\prime \in \mathbb{R}_{\ge 0}$, let
\begin{equation}
\mathbf{Z} = \left(\begin{array}{ccc}
 \cos(k_1  x) \sin(k_1^\prime x)   &  \cos(k_1  x) \cos(k_1^\prime x)   &  \sin(k_1  x)\\
\cos(k_2  x) \sin(k_2^\prime x)    &  \cos(k_2  x) \cos(k_2^\prime x)    &  \sin(k_2  x) \\
\vdots   & \vdots   & \vdots   \\
\cos(k_N  x) \sin(k_N^\prime x) & \cos(k_N  x)\cos(k_N^\prime x)  &  \sin(k_N  x) \\
\end{array} \right)
\end{equation}
and $\boldsymbol{\tilde{A}} = \mathbf{ZZ}^\top$.Then for any rank $3$ adjacency $\mathbf{A}$, there exists $k_i,k_i^\prime$ such that $\operatorname{sign}(\mathbf{A}) = \operatorname{sign}(\boldsymbol{\tilde{A}})$.
\end{proposition}

Propositions~\ref{firstproposition} and~\ref{firstpropgeneralized} provide a formal algorithm to generate graphs of ranks $2$ or $3$. Figure~\ref{ranks2and3} provides additional examples of low-rank graphs as well as parameters used. 

\begin{figure}[h!]
  \begin{subfigure}[b]{0.24\textwidth}
    \includegraphics[width=\textwidth]{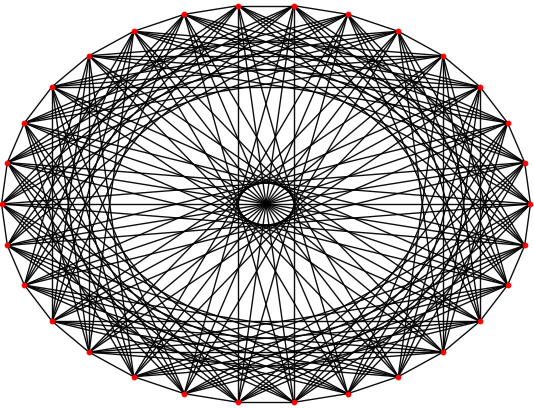}
    \caption{Rank $\le2$}
  \end{subfigure}
  \hfill
    \begin{subfigure}[b]{0.24\textwidth}
    \includegraphics[width=\textwidth]{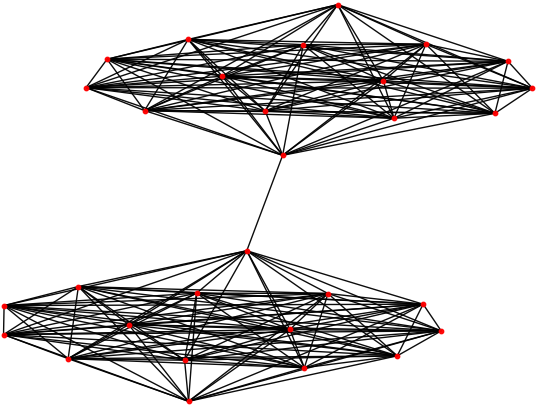}
    \caption{Rank $\le2$}
  \end{subfigure}
  \hfill
  \begin{subfigure}[b]{0.24\textwidth}
    \includegraphics[width=\textwidth]{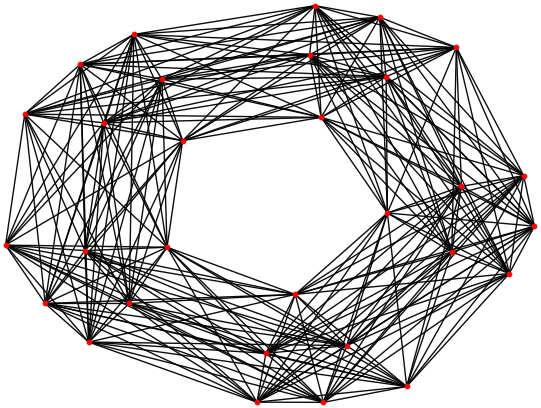}
    \caption{Rank $\le2$}
  \end{subfigure}
  \hfill
  \begin{subfigure}[b]{0.24\textwidth}
    \includegraphics[width=\textwidth]{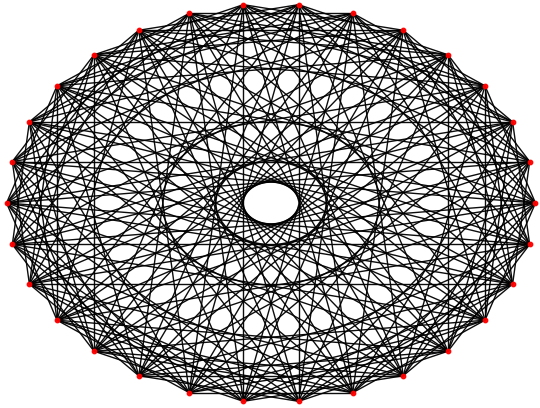}
    \caption{Rank $\le2$}
  \end{subfigure}
  \hfill
  \begin{subfigure}[b]{0.24\textwidth}
    \includegraphics[width=\textwidth]{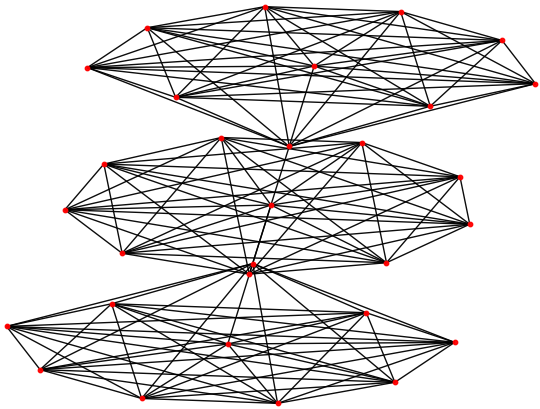}
    \caption{Rank $\le2$}
  \end{subfigure}
  \hfill
    \begin{subfigure}[b]{0.24\textwidth}
    \includegraphics[width=\textwidth]{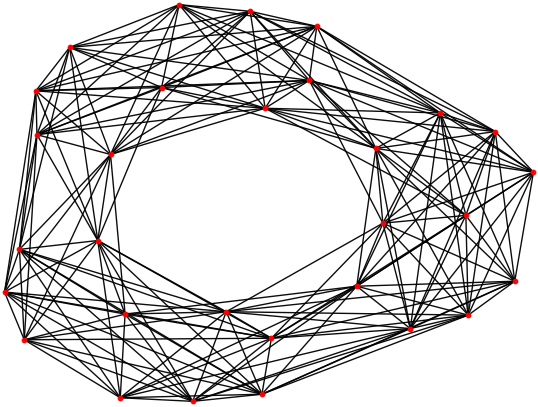}
    \caption{Rank $\le2$}
  \end{subfigure}
  \hfill
  \begin{subfigure}[b]{0.24\textwidth}
    \includegraphics[width=\textwidth]{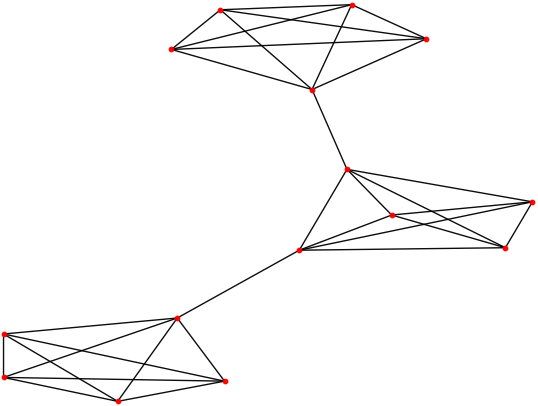}
    \caption{Rank $\le2$}
  \end{subfigure}
  \hfill
  \begin{subfigure}[b]{0.24\textwidth}
    \includegraphics[width=\textwidth]{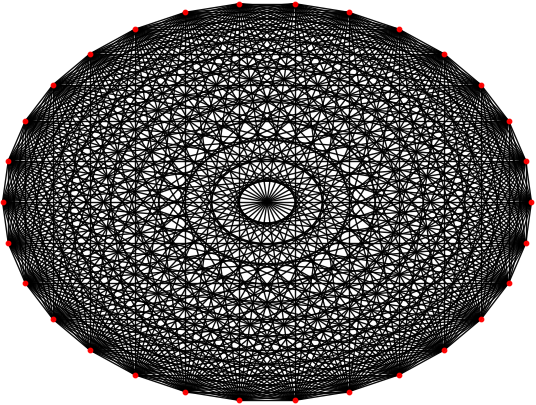}
    \caption{Rank $\le2$}
  \end{subfigure}
   \hfill
  \begin{subfigure}[b]{0.24\textwidth}
    \includegraphics[width=\textwidth]{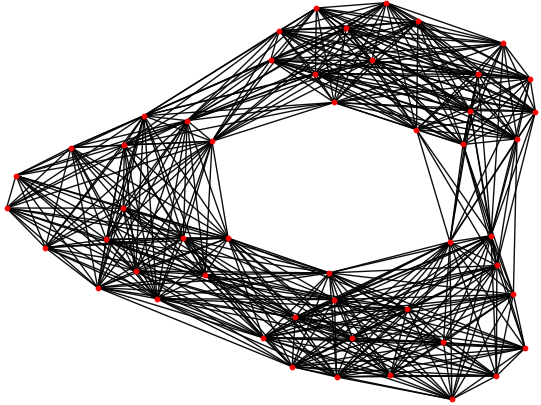}
    \caption{Rank $\le3$}
  \end{subfigure}
  \hfill
    \begin{subfigure}[b]{0.24\textwidth}
    \includegraphics[width=\textwidth]{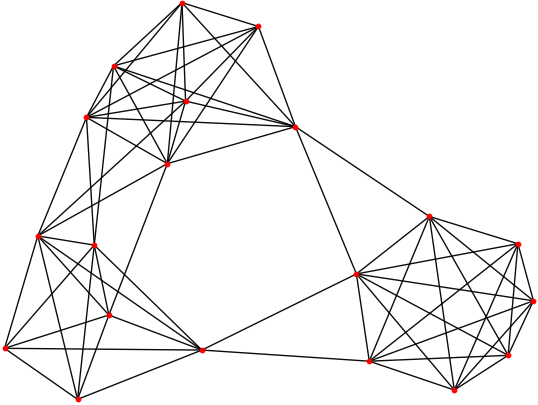}
    \caption{Rank $\le3$}
  \end{subfigure}
  \hfill
  \begin{subfigure}[b]{0.24\textwidth}
    \includegraphics[width=\textwidth]{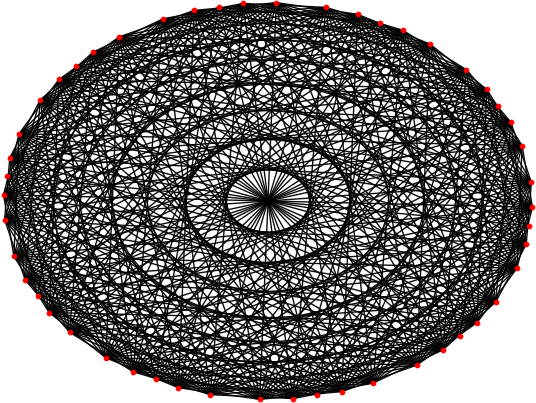}
    \caption{Rank $\le3$}
  \end{subfigure}
  \hfill
  \begin{subfigure}[b]{0.24\textwidth}
    \includegraphics[width=\textwidth]{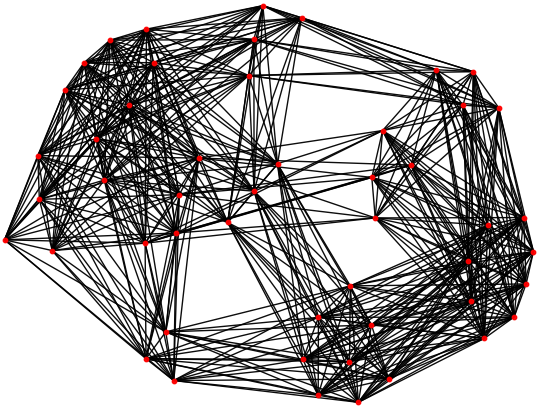}
    \caption{Rank $\le3$} 
  \end{subfigure}
   \hfill
  \begin{subfigure}[b]{0.24\textwidth}
    \includegraphics[width=\textwidth]{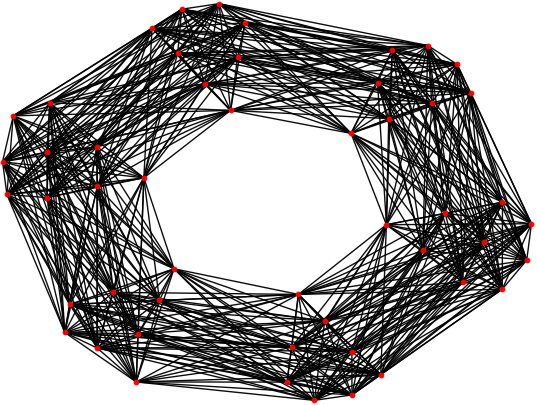}
    \caption{Rank $\le3$}
  \end{subfigure}
  \hfill
    \begin{subfigure}[b]{0.24\textwidth}
    \includegraphics[width=\textwidth]{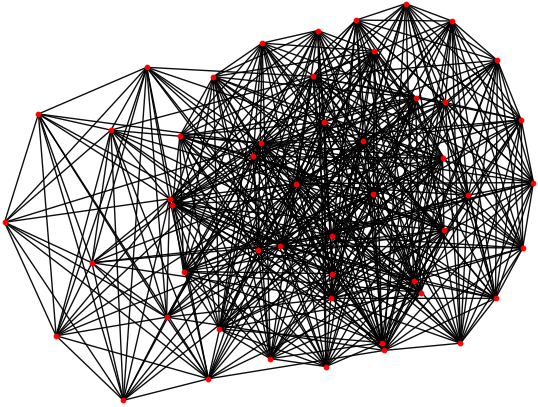}
    \caption{Rank $\le3$}
  \end{subfigure}
  \hfill
  \begin{subfigure}[b]{0.24\textwidth}
    \includegraphics[width=\textwidth]{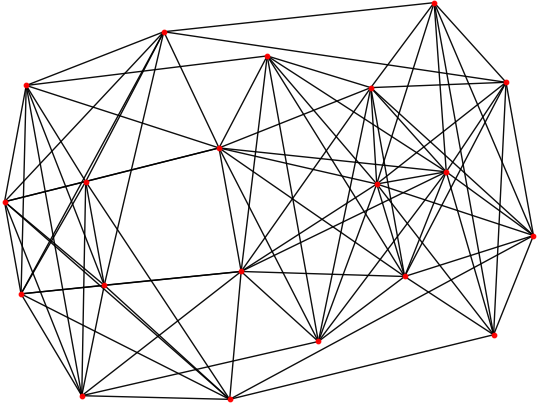}
    \caption{Rank $\le3$}
  \end{subfigure}
  \hfill
  \begin{subfigure}[b]{0.24\textwidth}
    \includegraphics[width=\textwidth]{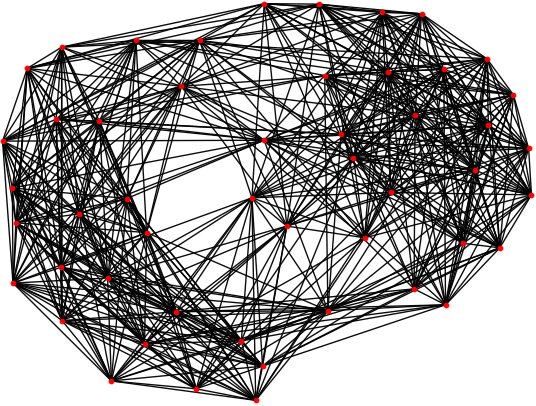}
    \caption{Rank $\le3$}
  \end{subfigure}
\caption{ Propositions~\ref{firstproposition} and~\ref{firstpropgeneralized} induce a straightforward algorithm that generates low-rank graphs given the input $(N, a, b, m)$. $N$ administers the total node count, and $x$ is varied over the endpoints of the $m-1$ equidistant interval partition of $[a,b]$. The parameters $k$ were chosen to be $k_i = i/N$ and $k_i^\prime = k_i/2$. The sign mapping was applied to $\tilde{\mathbf{A}}$ and the adjacency appended given an identical graph had not already been discovered. Afterwards, graphs displaying symmetric or otherwise intriguing node-relational attributes in the visualization were manually selected from the thousands of graphs generated. All rank $\le 2$ graphs in the main paper and supplement were collected from the pool $(30,0,50,1000)$, $(15,0,50,1000)$, $(50,0,100,2000)$, $(30,0,100,5000)$ while rank $\le 3$ graphs were chosen from $(15,0,50,1000)$, $(20,0,50,5000)$, $(10,0,150,5000)$, $(50,0,150,500)$. }
  \label{ranks2and3} 
\end{figure}

\section{Guaranteed lower bounds for graph rank}
In this section, we construct a class of graphs that cannot admit a representation of rank $2$. If all graphs are rank $2$, then for any adjacency $\mathbf{A}$, there must exist vectors $\mathbf{z}_1, \mathbf{z}_2, \mathbf{z}_3, \mathbf{z}_4 \in \mathbb{R}^N$ such that 
\begin{equation}
\operatorname{sign}(\mathbf{A}) = \operatorname{sign}(\mathbf{ZZ}^\top) = \operatorname{sign}\left(\left[\colvector{\mathbf{z}_1} \colvector{\mathbf{z}_2} \right] \times   \left[
\begin{array}{cc}
\rowvector{\mathbf{z}_3^\top} \\
\rowvector{\mathbf{z}_4^\top} 
\end{array}
\right] \right).    
\end{equation}
We prove that this cannot happen for $\mathbf{Z}$ real. 
\subsection{Not all graphs are of rank $2$}
\begin{theorem}
There exists a graph adjacency $\mathbf{A}$ such that
\begin{equation}
\operatorname{sign}(\mathbf{A}) \neq \operatorname{sign}\left(\left[\colvector{\mathbf{z}_1} \colvector{\mathbf{z}_2} \right] \times   \left[
\begin{array}{cc}
\rowvector{\mathbf{z}_3^\top} \\
\rowvector{\mathbf{z}_4^\top} 
\end{array}
\right] \right)  = \operatorname{sign}\left(\mathbf{RC} \right)    
\end{equation}
for $\mathbf{R}, \mathbf{C}^\top \in \mathbb{R}^{N\times 2}$.
\end{theorem}

\begin{proof}
As relabeling the nodes permutes the rows of $\mathbf{R}$ and columns of $\mathbf{C}$, we arbitrarily fix a permutation without loss of generality. Let $\mathbf{A}$ be such that the first column $\mathbf{A}[:,1]$ admits a singular $-$ only on the last entry and $+$ otherwise, that is,  $(+,\dots,+, -)^\top$. This ensures that the first column of $\mathbf{C}$ is nontrivial, $\mathbf{c}_1 \neq \overline{0}$. Now, let $\mathbf{A}[4:7,2]$ admit two $+$ and two $-$. For $\mathbf{r}_i$ the $i$-th row of $\mathbf{R}$, we denote nodes $i \in [4,7]$ such that $\langle \mathbf{r}_i,\mathbf{c}_1 \rangle > 0,$  $\langle \mathbf{r}_i,\mathbf{c}_2 \rangle > 0$ to be Group $1$, and $\langle \mathbf{r}_i,\mathbf{c}_1 \rangle > 0$, $\langle \mathbf{r}_i,\mathbf{c}_2 \rangle \le 0$ to be Group $2$. Choose precisely $1$ node each from Group $1$ and Group $2$, and assign $+$ to the corresponding entries in $\mathbf{A}[:,3]$. For the two nodes not chosen, we assign $-$. As $\mathbf{c}_1$ and $\mathbf{c}_2$ are distinct and non-zero, there cannot exist $\mathbf{c}_3$ which admits this sign combination (see Figure \ref{threekindsofnodesa} (a) caption for details). Now, choose the remaining signs of $\mathbf{A}$ so that no first seven rows or columns admit a singular sign excluding the diagonals, ensuring that none of $\mathbf{r}_i$ or $\mathbf{c}_i$ are $\overline{0}$ for $i = 1, \dots, 7$.
\end{proof}
\begin{remark}
Instead of learning $\operatorname{sign}(\mathbf{A})$, one may consider learning $-\operatorname{sign}(\mathbf{A})$ instead. But even in this case, if there exists $\mathbf{R},\mathbf{C}$ such that 
\begin{equation}
-\operatorname{sign}(\mathbf{A}) = \operatorname{sign}\left(\left[\colvector{\mathbf{z}_1} \colvector{\mathbf{z}_2} \right] \times   \left[
\begin{array}{cc}
\rowvector{\mathbf{z}_3^\top} \\
\rowvector{\mathbf{z}_4^\top} 
\end{array}
\right] \right)  = \operatorname{sign}\left(\mathbf{RC} \right)    
\end{equation}
for the $\mathbf{A}$ constructed above, then this must imply that 
\begin{equation}
\operatorname{sign}(\mathbf{A}) = \widetilde{\operatorname{sign}}\left(\left[\colvector{\mathbf{z}_1} \colvector{\mathbf{z}_2} \right] \times   \left[
\begin{array}{cc}
\rowvector{-\mathbf{z}_3^\top} \\
\rowvector{-\mathbf{z}_4^\top} 
\end{array}
\right] \right)  = \operatorname{sign}\left(\mathbf{R}(-\mathbf{C}) \right),    
\end{equation}
where $\widetilde{\operatorname{sign}}: \mathbb{R} \to \{+,-\}$ takes the negative sign on $\mathbb{R}_{<0}$. An analogous argument can be made to demonstrate the impossibility of such a result for general $\mathbf{A}$.
\end{remark}

\begin{figure*}[h!]
  \begin{minipage}[c]{0.14\textwidth}
    \begin{subfigure}[b]{\textwidth}
    \includegraphics[width=\textwidth]{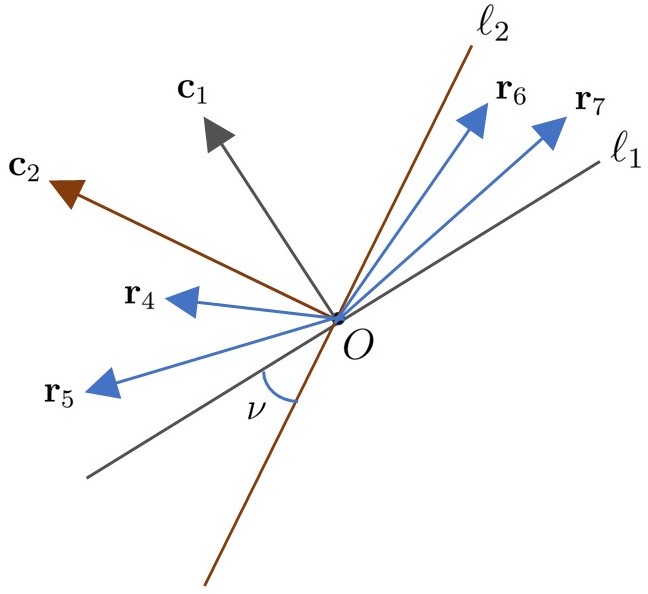}
    \caption{ }
  \end{subfigure}
    \begin{subfigure}[b]{\textwidth}
    \includegraphics[width=\textwidth]{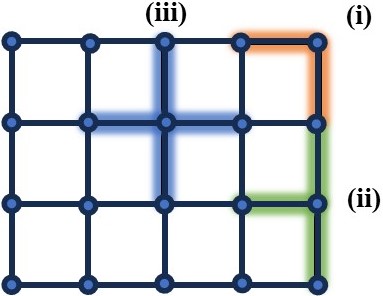}
    \caption{ }
  \end{subfigure}
  \begin{subfigure}[b]{\textwidth}
    \includegraphics[width=\textwidth]{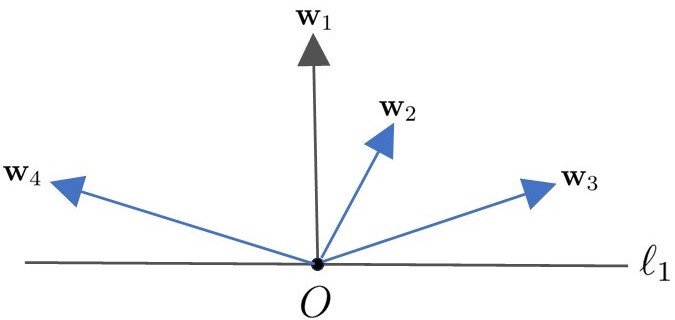}
    \caption{ }
  \end{subfigure}
  \end{minipage}\hfill
  \begin{minipage}[c]{0.79\textwidth}
  \captionsetup{singlelinecheck=off}
 \caption{(a) As $\mathbf{c}_1,\mathbf{c}_2$ are distinct and non-zero, the angle $\nu >0$ between $\ell_1,\ell_2$ ($\ell_i$ is the line perpendicular to $\mathbf{c}_i$ passing through the origin $O$) is nontrivial. $\mathbf{r}_4,\mathbf{r}_5$ are in Group $1$, $\mathbf{r}_6,\mathbf{r}_7$ in Group $2$. Suppose we have selected $\mathbf{r}_4,\mathbf{r}_6$ in the proof of Theorem \ref{beginningtheorem}, assigning $A_{43}=A_{63}=+$ to the corresponding entries in $\mathbf{A}[:,3]$ and $A_{53}=A_{73}=-$. There exists no line $\ell_3 \neq \ell_1 $ going through $O$ that leaves $\mathbf{r}_4,\mathbf{r}_6$ on one side of the hyperplane induced by $\ell_3$ while leaving $\mathbf{r}_5,\mathbf{r}_7$ on the other. Thus, no vector $\mathbf{c}_3$ can exist satisfying the imposed sign configuration as $\mathbf{c}_3$ is pairwise linearly independent to $\mathbf{c}_1$. (b) The three kinds of nodes appearing in a planar grid graph. The majority are of form \textbf{(iii)} as $N \to \infty$. (c) $\mathbf{w}_1$ corresponds to the center node in (b)-\textbf{(ii)}, whereas $\mathbf{w}_i$ for $i \in [2,4]$ depicts the neighboring nodes. For each of these nodes to be disconnected, any two distinct $\mathbf{w}_i,\mathbf{w}_j$ vectors for $i,j \in [2,4]$ must form an obtuse angle. However, the pidegonhole principle immediately gives that the existence of three such vectors $\mathbf{w}_i$ is sufficient to violate this requirement, as being connected with node $1$ constrains the three vectors to lying in the interior of a hyperplane induced by $\ell_1$. As the $3$-star graph may be identified as an induced subgraph of the $4$-star graph in (b)-\textbf{(iii)}, the $4$-star graph is rank greater than $2$. }
  \label{threekindsofnodesa}
  \end{minipage}
\end{figure*}

\subsection{Grid graphs are not rank $2$}
\begin{theorem}
Two dimensional planar grid graphs composed of more than one $4$-cycle cannot be rank $2$.    
\end{theorem}
\begin{proof}
To show that a grid graph is not of rank $2$, we decompose the grid into subgraphs, where we allow for nodes and edges to be repeatedly represented in multiple subgraphs. We need only impose that all edges connecting any two subgraph nodes in the original graph must be preserved in each subgraph. If any of the subgraph structures do not admit a representation of rank $2$, the entire graph cannot be represented by rank $2$. In particular, we will show that the $3$-star graph shown in Figure \ref{threekindsofnodesa} (b)-\textbf{(iii)} possesses this property. For contradiction, assume that there exists a real $\mathbf{Z}$ of latent dimension $f = 2$ such that $\operatorname{sign}(\mathbf{A}) = \operatorname{sign}(\mathbf{ZZ}^\top)$. We note that none of the \textit{row} vectors of $\mathbf{Z}$, $\mathbf{w}_i$, can be $\overline{0}$ as no node is fully disconnected with all other nodes. Let the central node in \textbf{(iii)} be represented by $\mathbf{w}_1$. Then the other three nodes $\mathbf{w}_2$, $\mathbf{w}_3$, $\mathbf{w}_4$ must be bounded by the \textit{interior} of a hyperplane separated by $\ell_1$ (see Figure \ref{threekindsofnodesa} (c) caption), which is impossible due to the pigeonhole principle.    
\end{proof}

\subsection{Complexification of latent space for star graph encoding}
\begin{theorem}
Any star graph is of complex rank $1$.    
\end{theorem} 
\begin{proof}
A construction is given in Figure~\ref{starconstruction}.    
\end{proof}

\begin{figure*}[h!]
  \begin{minipage}[c]{0.17\textwidth}
    \includegraphics[width=\textwidth]{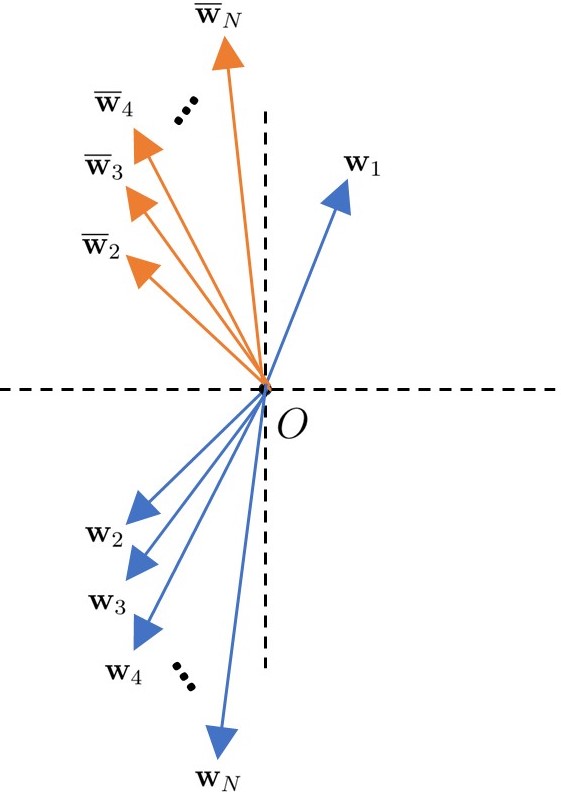}
  \end{minipage}\hfill
  \begin{minipage}[c]{0.79\textwidth}
  \captionsetup{singlelinecheck=off}
  \caption{For $\mathbf{z}_r,$ $\mathbf{z}_{im} \in \mathbb{R}^N$ the real and imaginary components of $\mathbf{z} \in \mathbb{C}^N$,\\$\operatorname{sign}(\mathbf{A}) = \operatorname{sign}\left(\mathfrak{Re}\left(\mathbf{zz}^T\right) \right) = \operatorname{sign}\left(\left[\colvector{\mathbf{z}_r} \colvector{\mathbf{z}_{im}} \right] \times   \left[\begin{array}{cc} \rowvector{\mathbf{z}_r^\top} \\ \rowvector{-\mathbf{z}_{im}^\top} \end{array} \right] \right) = \operatorname{sign}\left(\mathbf{RC}\right). $ \\ Denote the $i$-th row of $\mathbf{R}$ as $\mathbf{w}_i$, and its reflection in the axis of the first coordinate to be $\overline{\mathbf{w}}_i$. The representation of an $(N-1)$-star graph can be constructed in $\mathbb{R}^2$ as shown in the left. $\mathbf{w}_2$ forms a $\pi/4$ angle with the horizontal axis.}
  \label{starconstruction}
  \end{minipage}
\end{figure*}

\section{Hyperparameter Selection}

\subsection{On minimal $\ell_2$ regularization}

While being robust to alternations in the hidden dimension $h_1$ (Figure~\ref{largeplot} (b,f)), the architectures presented are more sensitive to $\ell_2$ regularization. Figure~\ref{largeplot} (a-e) demonstrates the effect of varying the regularization rate $\lambda = 1, 10^{-7}, 10^{-14}$, as well as using the squared Frobenius norm in the loss. Increasing regularization inhibits graph memorization as the backpropagation is less able to respect the insignificance of the binary encoding in the graph adjacency. In fact, regularization was utilized only to forbid entries of $\tilde{\mathbf{A}}$ from diverging in magnitude during training. We note also that choosing a smaller learning rate $\gamma = 10^{-4}$ enables more precise memorization than higher learning rates (e.g. Figure~\ref{largeplot} (a,e)). 

\begin{figure}[h!]
  \begin{subfigure}[b]{0.32\textwidth}
    \includegraphics[width=\textwidth]{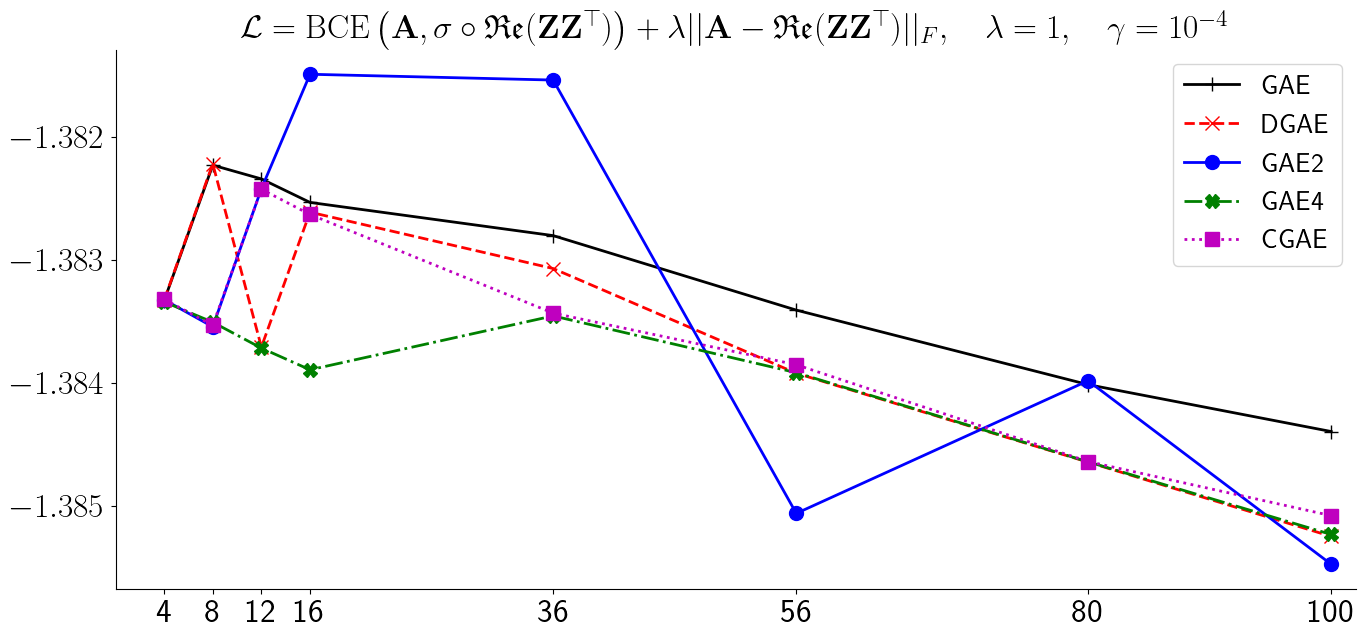}
    \caption{$\lambda = 1$}
  \end{subfigure}
  \hfill
    \begin{subfigure}[b]{0.32\textwidth}
    \includegraphics[width=\textwidth]{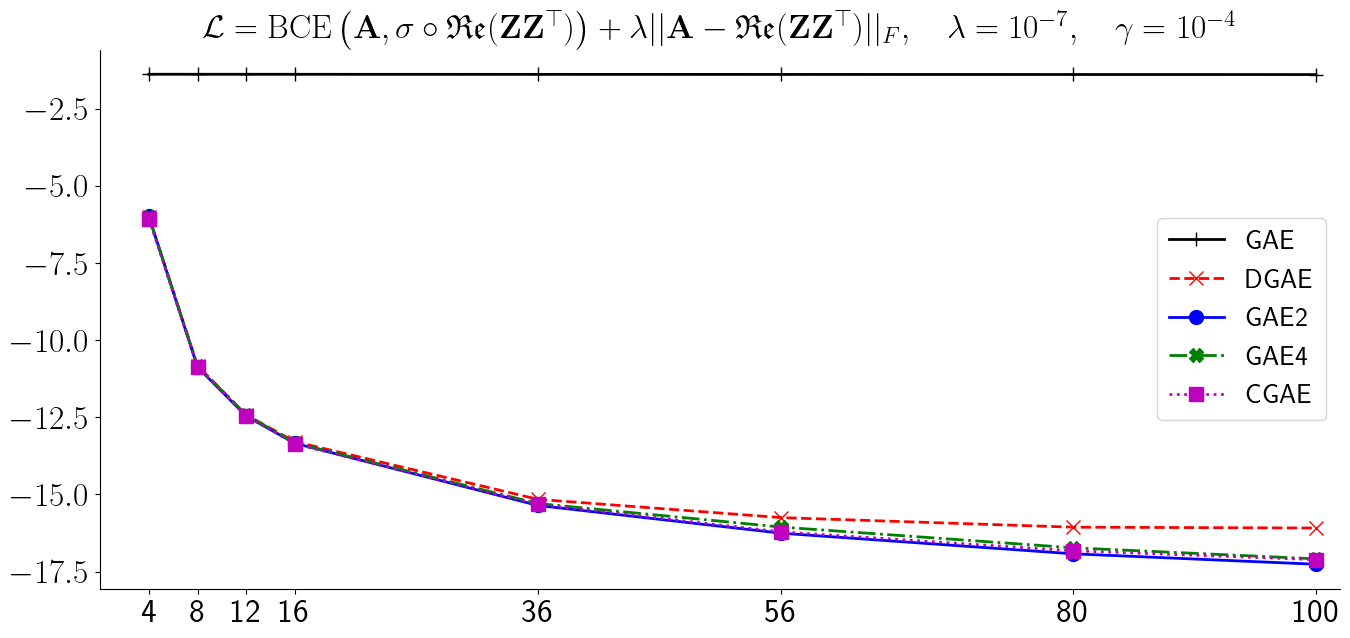}
    \caption{$\lambda = 10^{-7}$}
  \end{subfigure}
  \hfill
    \begin{subfigure}[b]{0.32\textwidth}
    \includegraphics[width=\textwidth]{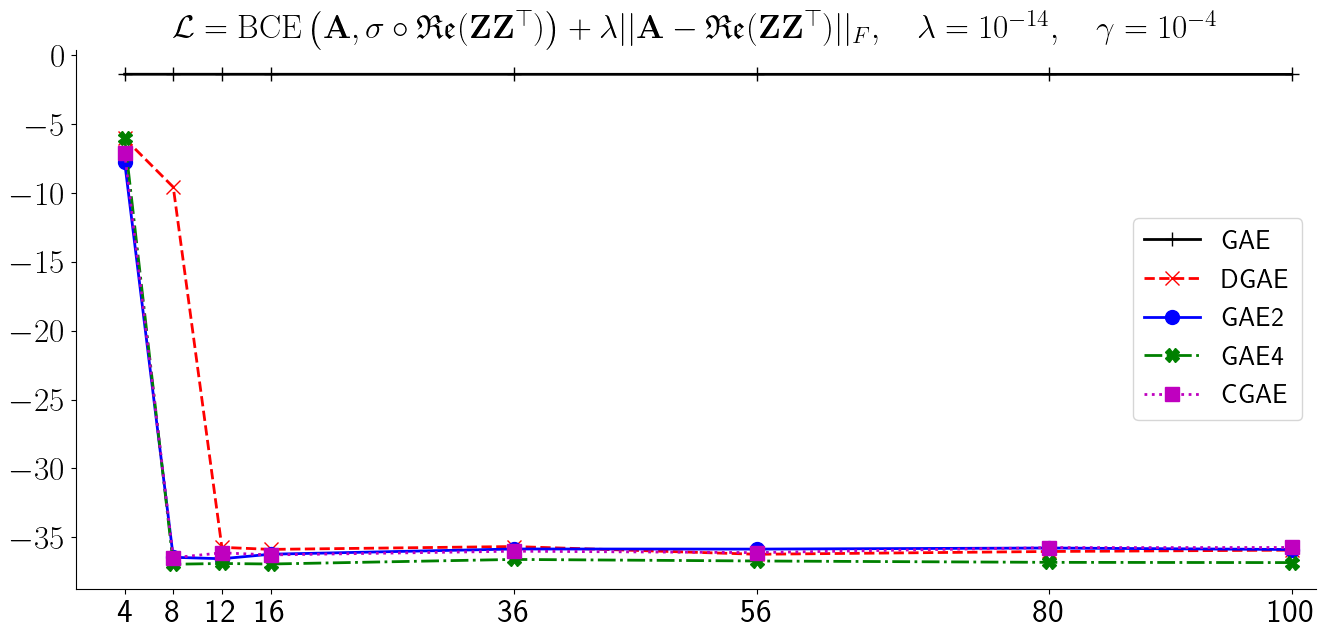}
    \caption{$\lambda = 10^{-14}$}
  \end{subfigure}
  \hfill
    \begin{subfigure}[b]{0.32\textwidth}
    \includegraphics[width=\textwidth]{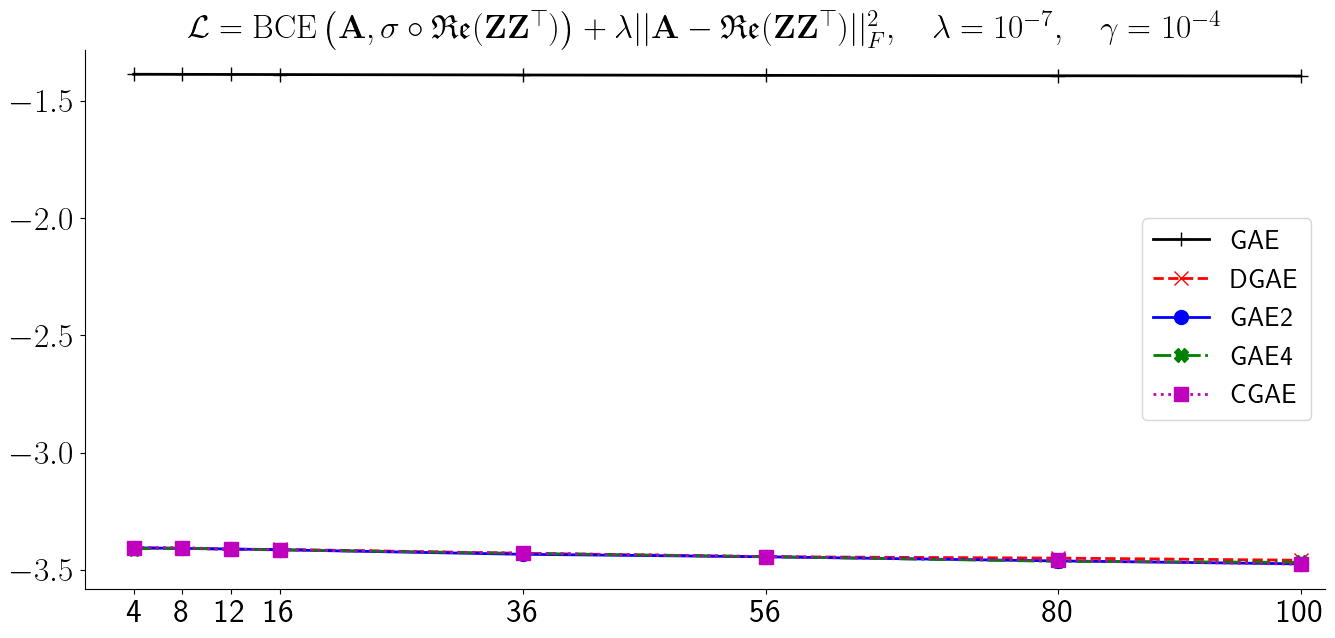}
    \caption{Norm squared loss}
  \end{subfigure}
  \hfill
    \begin{subfigure}[b]{0.32\textwidth}
    \includegraphics[width=\textwidth]{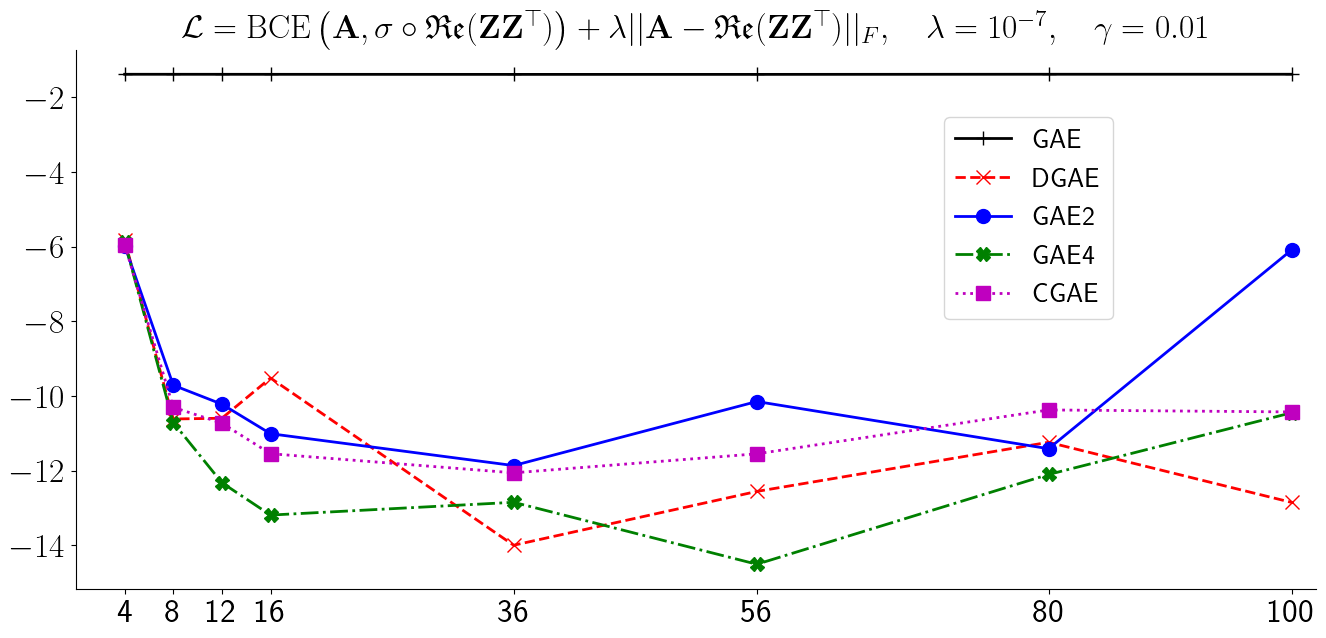}
    \caption{$\gamma = 10^{-2}$}
  \end{subfigure}
  \hfill
    \begin{subfigure}[b]{0.32\textwidth}
    \includegraphics[width=\textwidth]{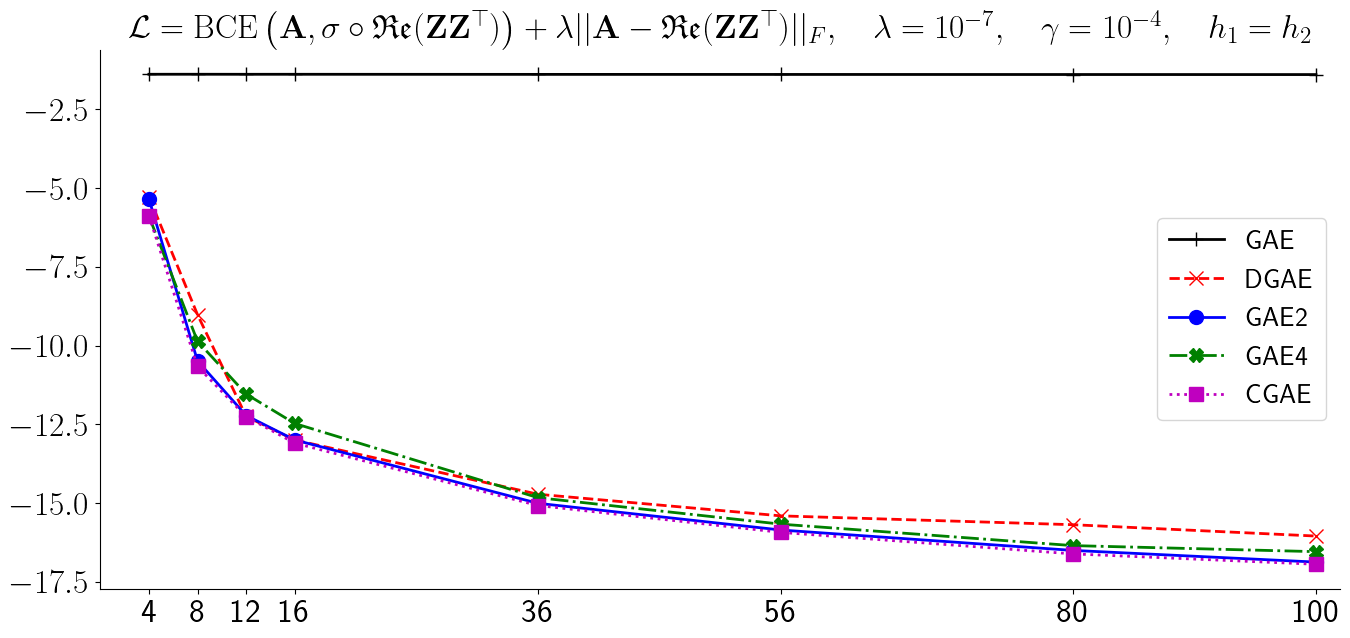}
    \caption{$h_1 = h_2$}
  \end{subfigure}
  \hfill
    \begin{subfigure}[b]{0.24\textwidth}
    \includegraphics[width=\textwidth]{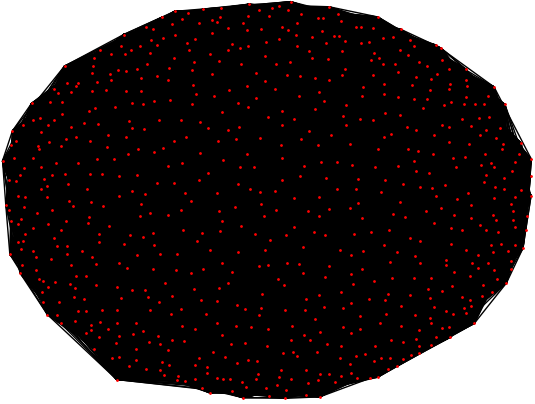}
    \caption{GAE, $h_2 = 597$}
  \end{subfigure}
  \hfill
    \begin{subfigure}[b]{0.24\textwidth}
    \includegraphics[width=\textwidth]{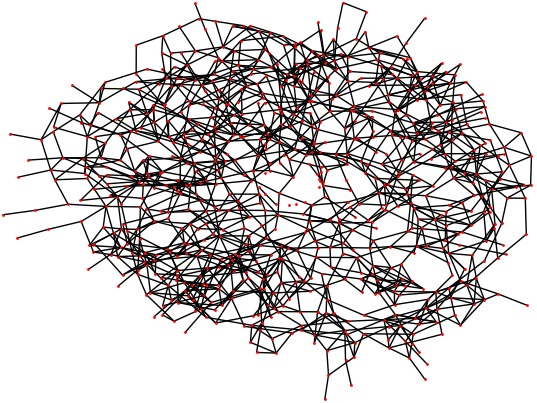}
    \caption{DGAE, $h_2 = 4$}
  \end{subfigure}
  \hfill
  \begin{subfigure}[b]{0.24\textwidth}
    \includegraphics[width=\textwidth]{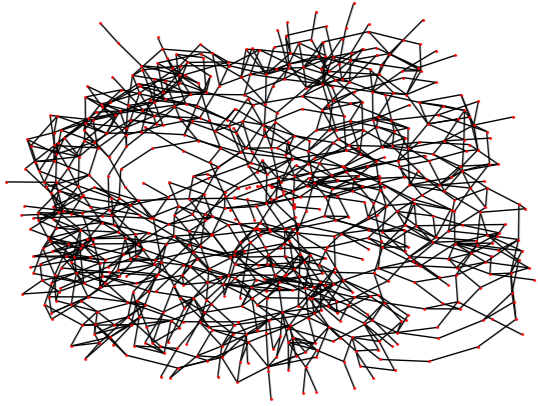}
    \caption{2GAE, $h_2 = 4$}
  \end{subfigure}
  \hfill
    \begin{subfigure}[b]{0.24\textwidth}
    \includegraphics[width=\textwidth]{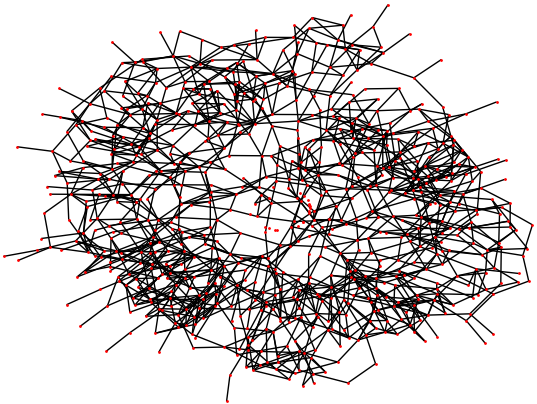}
    \caption{CGAE, $h_2 = 4$}
  \end{subfigure}
  \hfill
  \begin{subfigure}[b]{0.24\textwidth}
    \includegraphics[width=\textwidth]{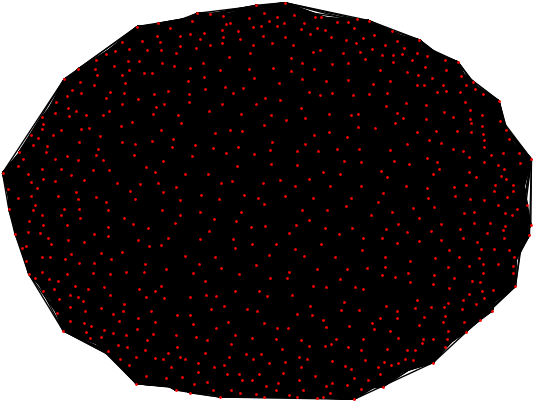}
    \caption{GAE, $h_2 = 597$}
  \end{subfigure}
  \hfill
  \begin{subfigure}[b]{0.24\textwidth}
    \includegraphics[width=\textwidth]{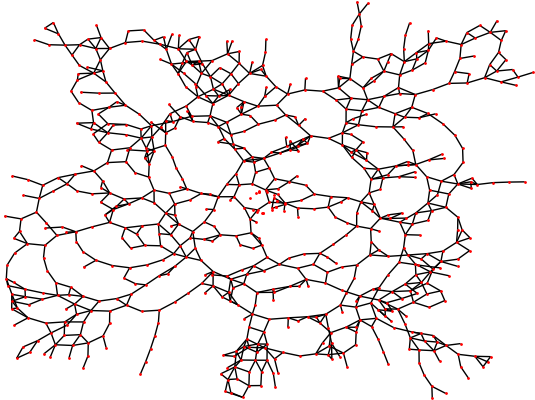}
    \caption{DGAE, $h_2 = 4$}
  \end{subfigure}
  \hfill
  \begin{subfigure}[b]{0.24\textwidth}
    \includegraphics[width=\textwidth]{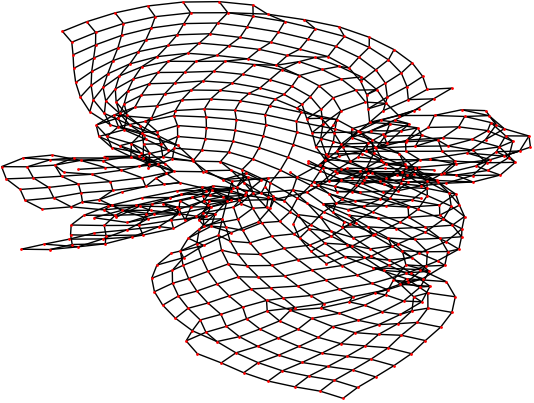}
    \caption{2GAE, $h_2 = 4$}
  \end{subfigure}
  \hfill
    \begin{subfigure}[b]{0.24\textwidth}
    \includegraphics[width=\textwidth]{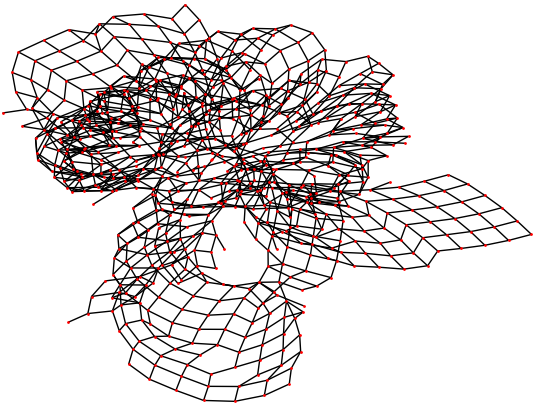}
    \caption{CGAE, $h_2 = 4$}
  \end{subfigure}
\caption{All numerical plots display the log-normalized distance $\log\left(d(\mathbf{A}, \tilde{\mathbf{A}})\right) = \log\left(||\mathbf{A}- \sigma(\tilde{\mathbf{A}})||_{F}^2/N^2\right)$ against latent dimension $h_2$. Training is done for $30000$ epochs on NVIDIA GeForce RTX 3060 GPU. (a-c) varies the regularization $\lambda = 1, 10^{-7}, 10^{-14}$ and observes improved memorization capabilities for small $\lambda$. Examples of probabilistically reconstructed graphs from latent projections in (b,c) are given in (g-j) ($\lambda = 10^{-7}$) and (k-n) ($\lambda = 10^{-14}$). (d) strengthens regularization drastically by squaring the norm, and (e) increases the learning rate by a factor of $\mathcal{O}(10^2)$. The setting in (f) is identical to (b) except that first layer hidden dimensions in all encoders were set to be equal to the latent dimensions each encoder projects to, while $h_1 = 120$ was taken for (a-e). The $27 \times 27$ long grid was chosen as a representative (see also Figure~\ref{lastplot}) but the trends depicted generalize consistently to all grid/chain graphs tested. As all test graphs in the paper are featureless, the latent embedding was learned from the one-hot identity feature $\mathbf{X} = \mathbf{I}$ which trivializes node relational information.}
  \label{largeplot}
\end{figure}

\begin{figure}[h!]
  \begin{subfigure}[b]{0.33\textwidth}
    \includegraphics[width=\textwidth]{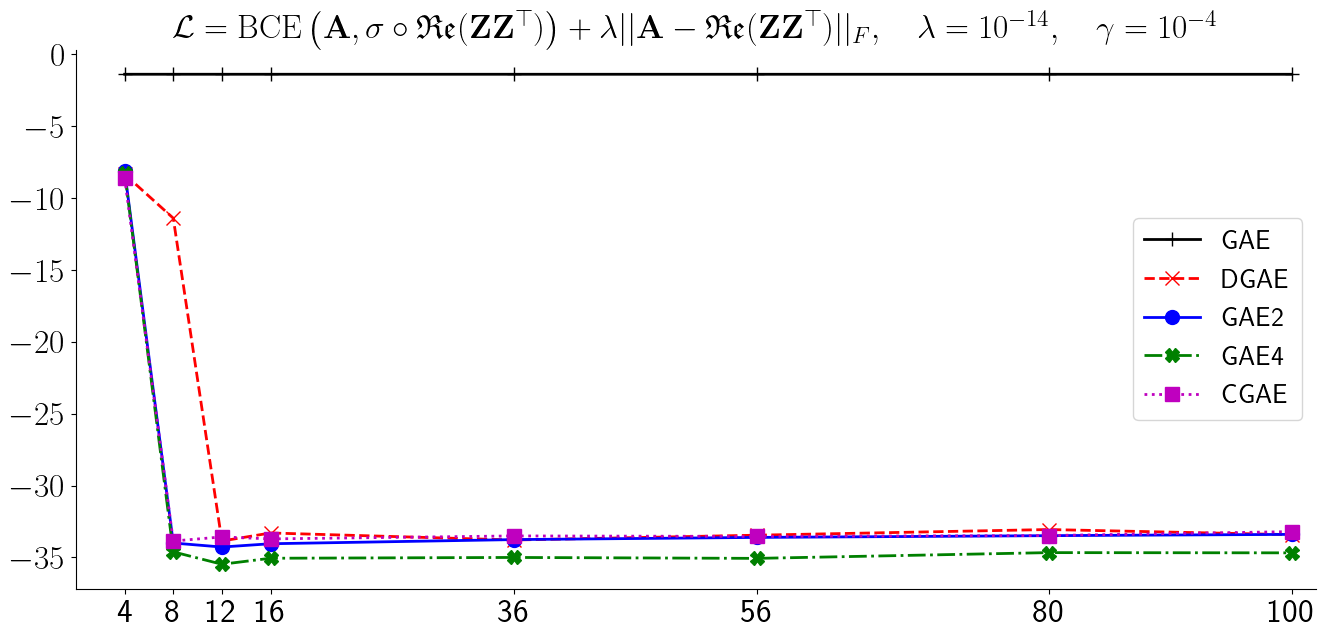}
    \caption{Log-normal distance}
  \end{subfigure}
  \hfill
    \begin{subfigure}[b]{0.33\textwidth}
    \includegraphics[width=\textwidth]{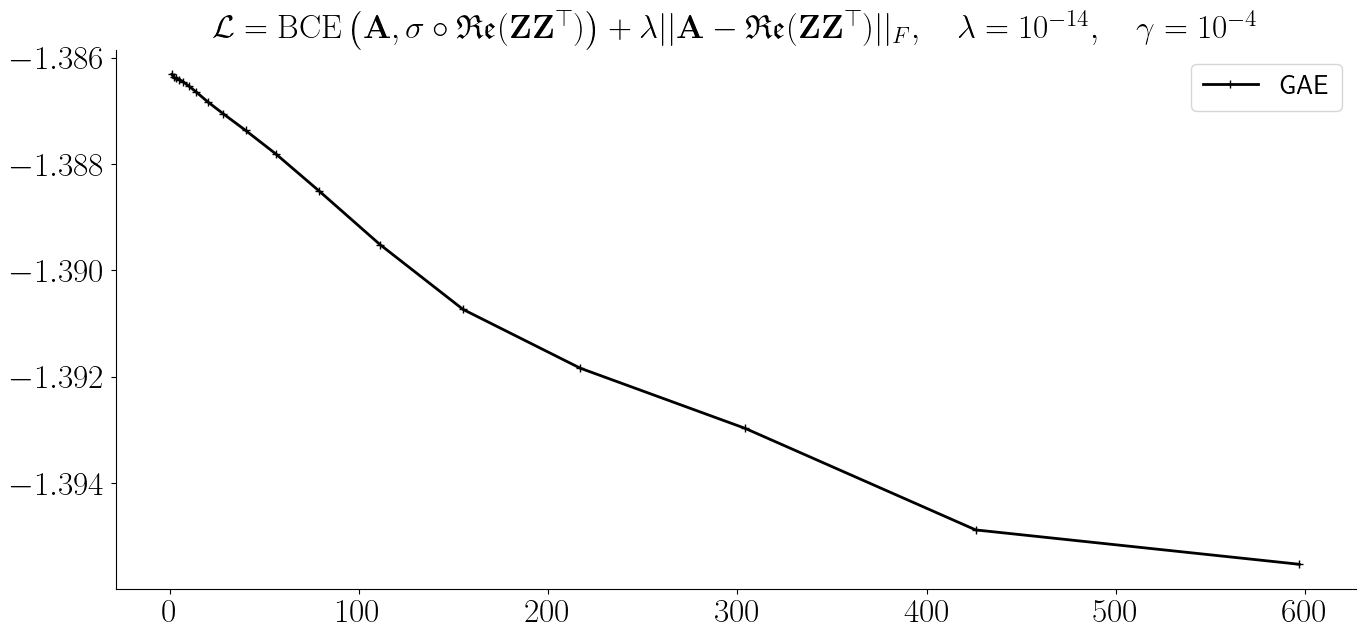}
    \caption{Log-normal distance}
  \end{subfigure}
    \hfill
    \begin{subfigure}[b]{0.23\textwidth}
    \includegraphics[width=\textwidth]{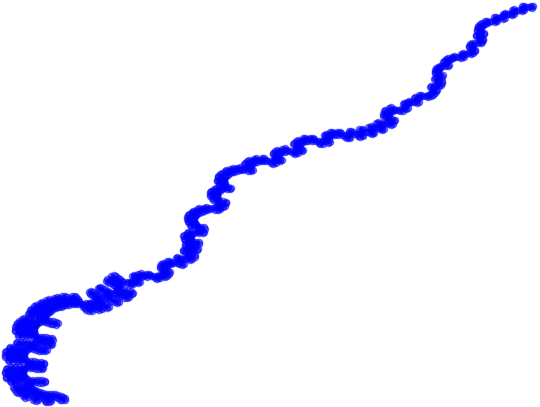}
    \caption{Original}
  \end{subfigure}
  \hfill
    \begin{subfigure}[b]{0.24\textwidth}
    \includegraphics[width=\textwidth]{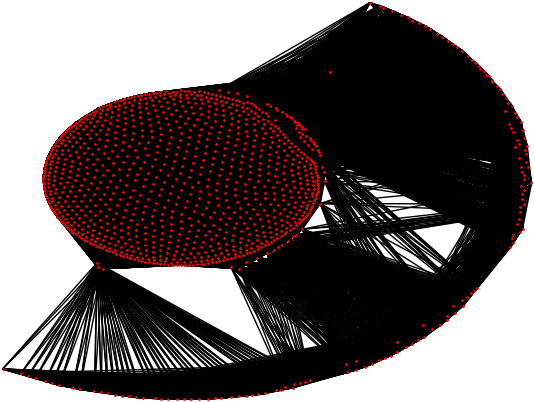}
    \caption{VAE, $h_2 = 4$}
  \end{subfigure}
  \hfill
    \begin{subfigure}[b]{0.24\textwidth}
    \includegraphics[width=\textwidth]{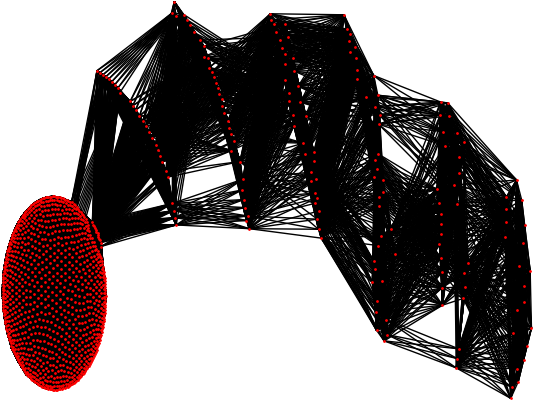}
    \caption{VAE, $h_2 = 8$}
  \end{subfigure}
  \hfill
    \begin{subfigure}[b]{0.24\textwidth}
    \includegraphics[width=\textwidth]{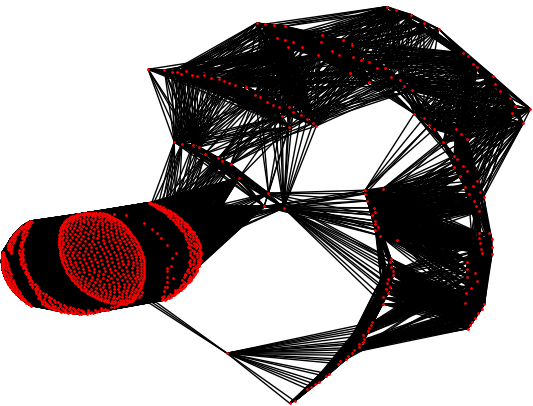}
    \caption{VAE, $h_2 = 12$}
  \end{subfigure}
  \hfill
    \begin{subfigure}[b]{0.24\textwidth}
    \includegraphics[width=\textwidth]{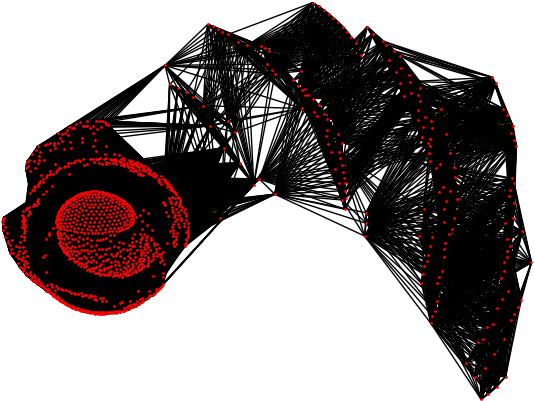}
    \caption{VAE, $h_2 = 16$}
  \end{subfigure}
  \hfill
  \begin{subfigure}[b]{0.24\textwidth}
    \includegraphics[width=\textwidth]{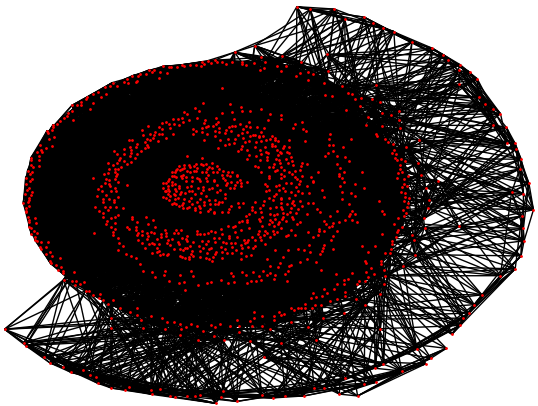}
    \caption{VAE, $h_2 = 217$}
  \end{subfigure}
  \hfill
    \begin{subfigure}[b]{0.24\textwidth}
    \includegraphics[width=\textwidth]{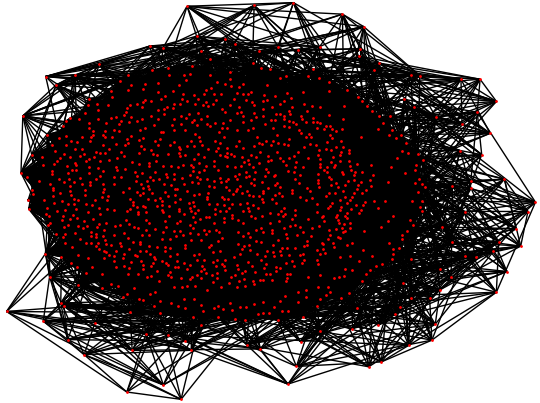}
    \caption{VAE, $h_2 = 304$}
  \end{subfigure}
  \hfill
  \begin{subfigure}[b]{0.24\textwidth}
    \includegraphics[width=\textwidth]{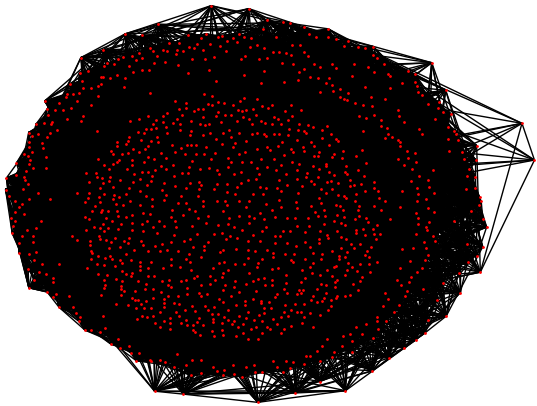}
    \caption{VAE, $h_2 = 426$}
  \end{subfigure}
  \hfill
  \begin{subfigure}[b]{0.24\textwidth}
    \includegraphics[width=\textwidth]{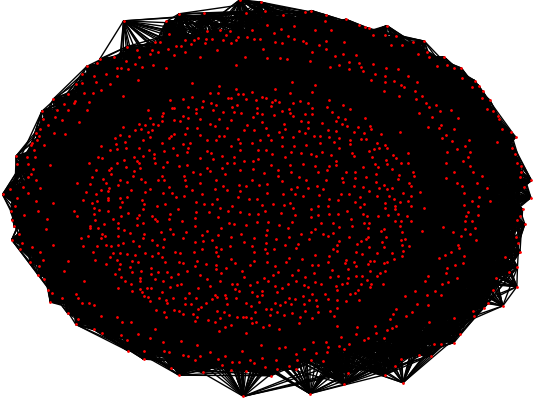}
    \caption{VAE, $h_2 = 597$}
  \end{subfigure}
  \hfill
  \begin{subfigure}[b]{0.24\textwidth}
    \includegraphics[width=\textwidth]{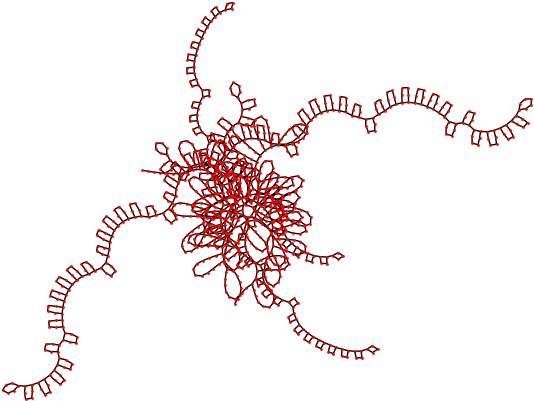}
    \caption{DVAE, $h_2 = 8$}
  \end{subfigure}
  \hfill
    \begin{subfigure}[b]{0.24\textwidth}
    \includegraphics[width=\textwidth]{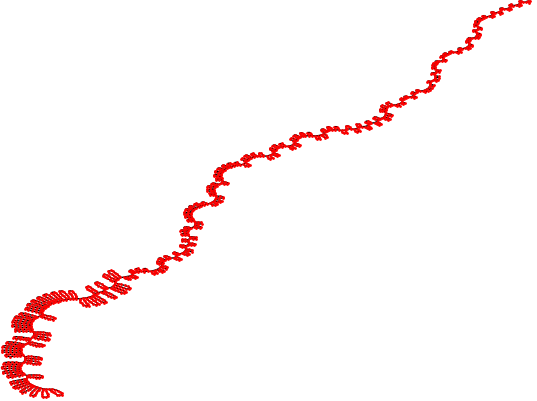}
    \caption{2VAE, $h_2 = 8$}
  \end{subfigure}
  \hfill
  \begin{subfigure}[b]{0.24\textwidth}
    \includegraphics[width=\textwidth]{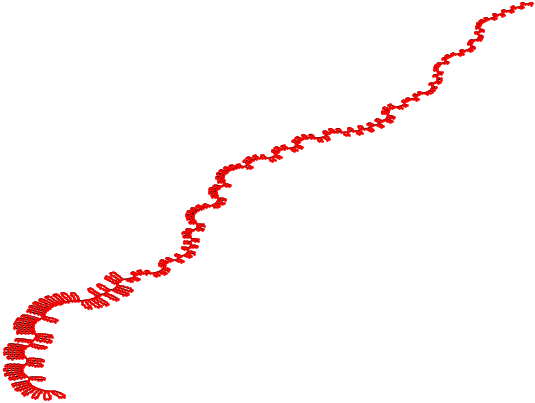}
    \caption{4VAE, $h_2 = 8$}
  \end{subfigure}
  \hfill
  \begin{subfigure}[b]{0.24\textwidth}
    \includegraphics[width=\textwidth]{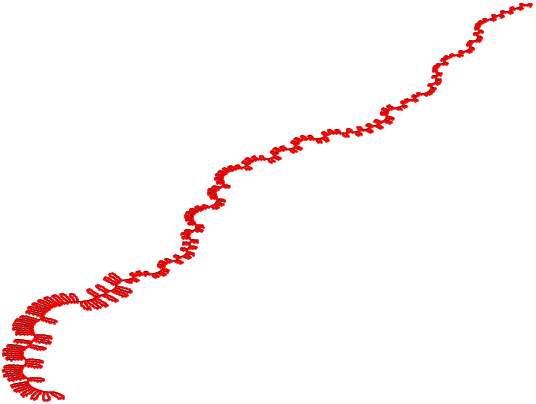}
    \caption{CVAE, $h_2 = 8$}
  \end{subfigure}
   \hfill
  \begin{subfigure}[b]{0.24\textwidth}
    \includegraphics[width=\textwidth]{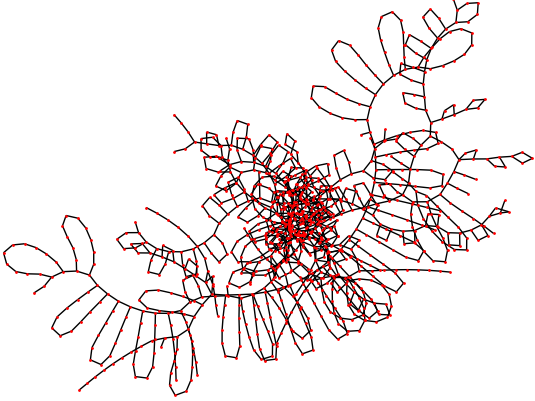}
    \caption{DVAE, $h_2 = 4$}
  \end{subfigure}
  \hfill
    \begin{subfigure}[b]{0.24\textwidth}
    \includegraphics[width=\textwidth]{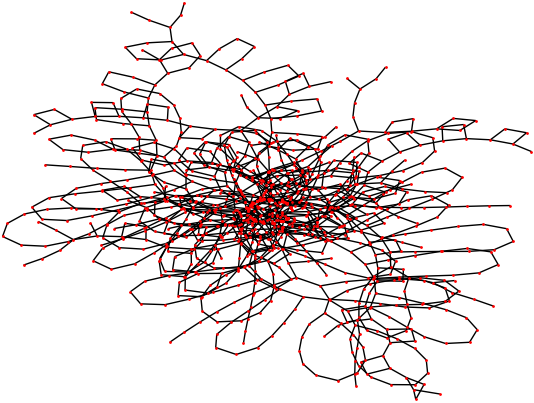}
    \caption{2VAE, $h_2 = 4$}
  \end{subfigure}
  \hfill
  \begin{subfigure}[b]{0.24\textwidth}
    \includegraphics[width=\textwidth]{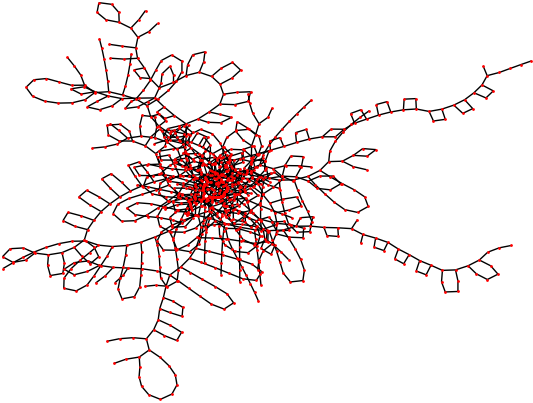}
    \caption{4VAE, $h_2 = 4$}
  \end{subfigure}
  \hfill
  \begin{subfigure}[b]{0.24\textwidth}
    \includegraphics[width=\textwidth]{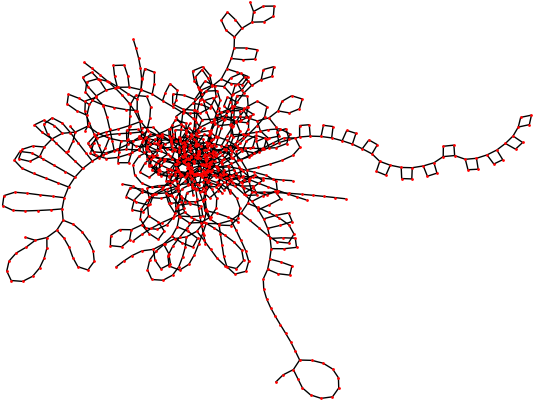}
    \caption{CVAE, $h_2 = 4$} 
  \end{subfigure}
   \hfill
  \begin{subfigure}[b]{0.24\textwidth}
    \includegraphics[width=\textwidth]{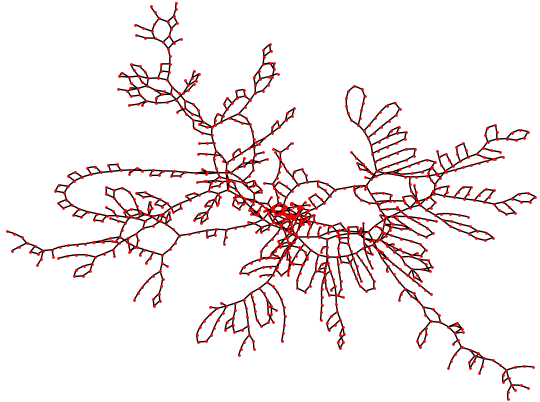}
    \caption{DVAE, $h_2 = 4$}
  \end{subfigure}
  \hfill
    \begin{subfigure}[b]{0.24\textwidth}
    \includegraphics[width=\textwidth]{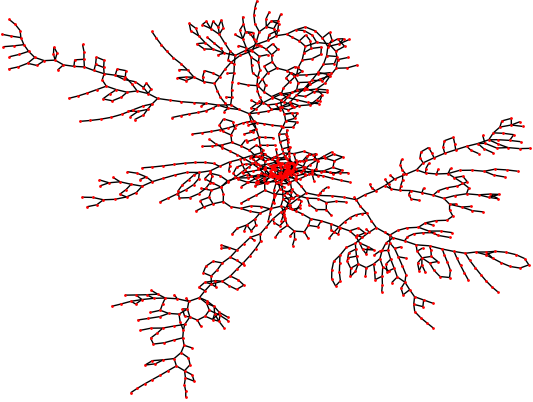}
    \caption{2VAE, $h_2 = 4$}
  \end{subfigure}
  \hfill
  \begin{subfigure}[b]{0.24\textwidth}
    \includegraphics[width=\textwidth]{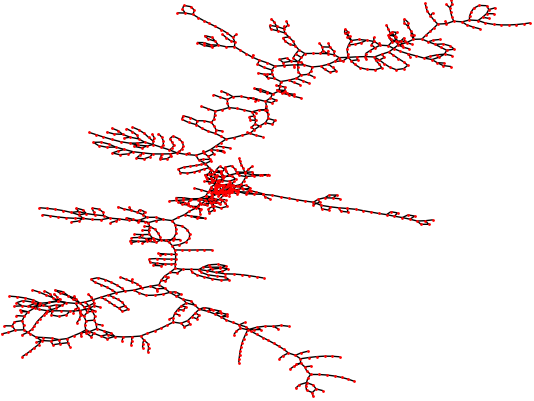}
    \caption{4VAE, $h_2 = 4$}
  \end{subfigure}
  \hfill
  \begin{subfigure}[b]{0.24\textwidth}
    \includegraphics[width=\textwidth]{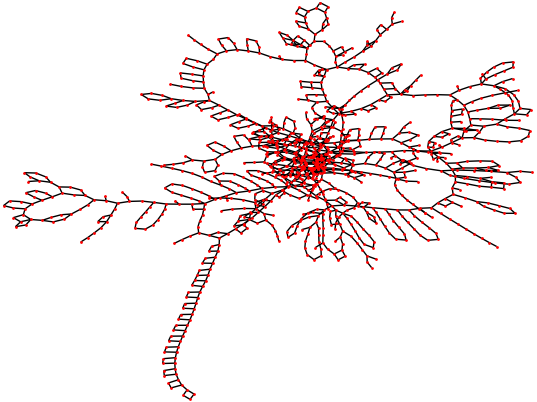}
    \caption{CVAE, $h_2 = 4$}
  \end{subfigure}
\caption{The long diverse chain (c) is comprised of linking $45$ $4$-cycles, $80$ $6$-cycles, and $45$ $12$-cycles totalling $1200$ nodes. Training is performed under identical hyperparameter settings as Figure~\ref{largeplot} (c), and the VAE setting in Figure $6$ (main text) for (b),(h-k) only, i.e. $h_2 = \lfloor{1.4}^n\rfloor$ for $n = 2:19$ for $h_1$ node count. (l-o),(t-w) give probabilistic reconstructions whereas (d-k),(p-s) give deterministic sign reconstructions based on the low-dimensional latent embedding. In particular, we note the contrast between the reconstruction capabilities illuminated in (h-k) and (l-o). }
\end{figure}

\begin{figure}[h!]
      \begin{subfigure}[b]{0.24\textwidth}
    \includegraphics[width=\textwidth]{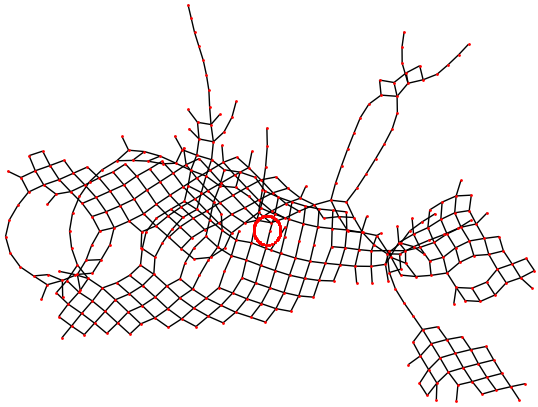}
    \caption{DVAE, $h_2 = 4$}
  \end{subfigure}
  \hfill
      \begin{subfigure}[b]{0.24\textwidth}
    \includegraphics[width=\textwidth]{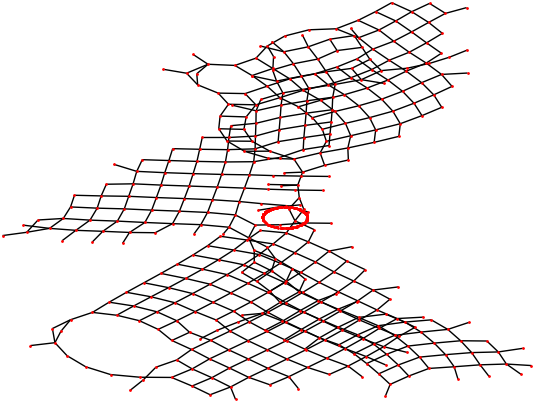}
    \caption{2VAE, $h_2 = 4$}
  \end{subfigure}
  \hfill
      \begin{subfigure}[b]{0.24\textwidth}
    \includegraphics[width=\textwidth]{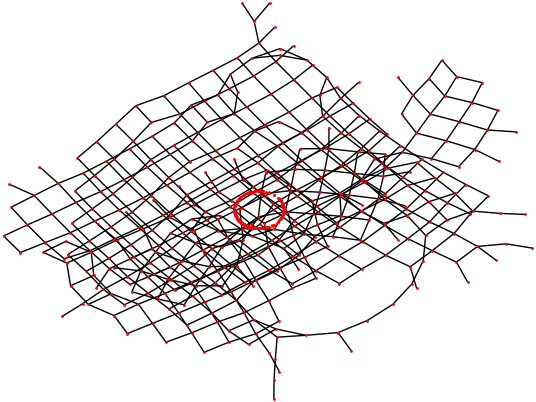}
    \caption{4VAE, $h_2 = 4$}
  \end{subfigure}
  \hfill
      \begin{subfigure}[b]{0.24\textwidth}
    \includegraphics[width=\textwidth]{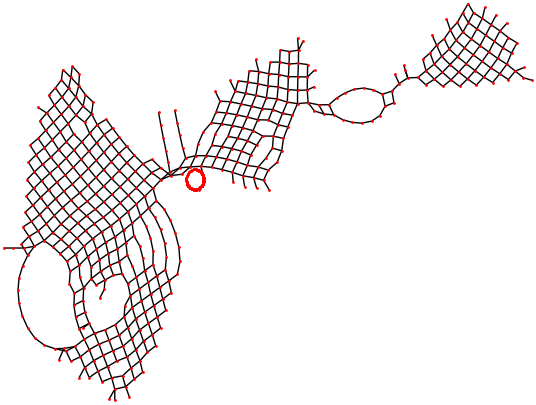}
    \caption{CVAE, $h_2 = 4$}
  \end{subfigure}
  \hfill
      \begin{subfigure}[b]{0.24\textwidth}
    \includegraphics[width=\textwidth]{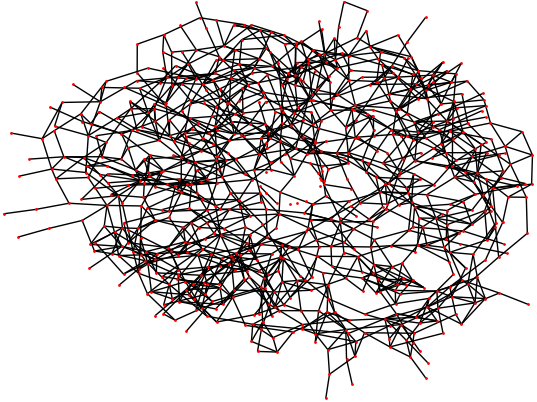}
    \caption{DVAE, $h_2 = 4$}
  \end{subfigure}
  \hfill
      \begin{subfigure}[b]{0.24\textwidth}
    \includegraphics[width=\textwidth]{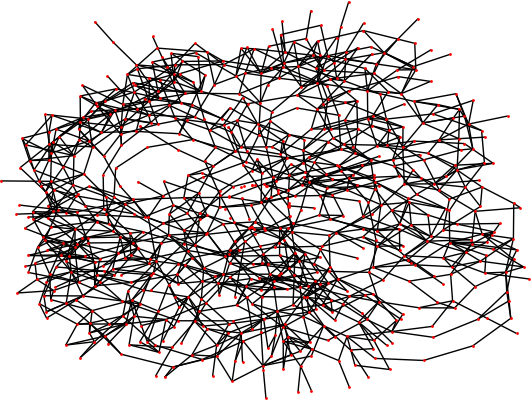}
    \caption{2VAE, $h_2 = 4$}
  \end{subfigure}
  \hfill
      \begin{subfigure}[b]{0.24\textwidth}
    \includegraphics[width=\textwidth]{ch2GAEr4prob.png}
    \caption{4VAE, $h_2 = 4$}
  \end{subfigure}
  \hfill
      \begin{subfigure}[b]{0.24\textwidth}
    \includegraphics[width=\textwidth]{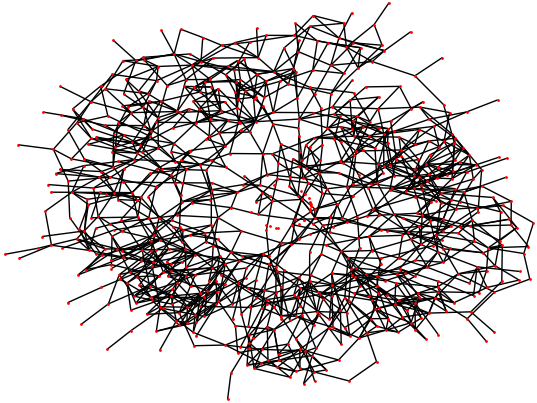}
    \caption{CVAE, $h_2 = 4$}
  \end{subfigure}
  \hfill
    \begin{subfigure}[b]{0.24\textwidth}
    \includegraphics[width=\textwidth]{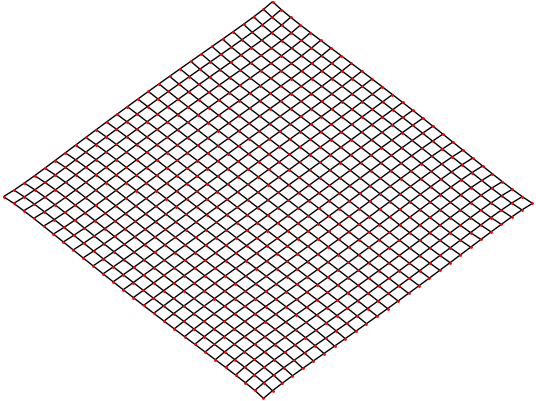}
    \caption{DVAE, $h_2 = 8$}
  \end{subfigure}
  \hfill
    \begin{subfigure}[b]{0.24\textwidth}
    \includegraphics[width=\textwidth]{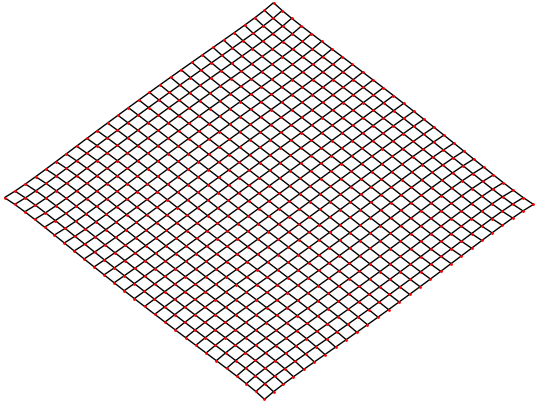}
    \caption{2VAE, $h_2 = 8$}
  \end{subfigure}
  \hfill
  \begin{subfigure}[b]{0.24\textwidth}
    \includegraphics[width=\textwidth]{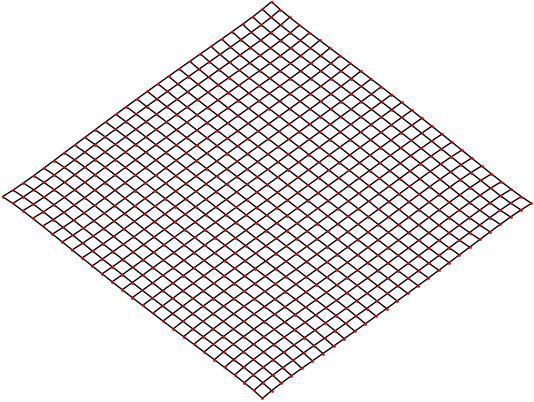}
    \caption{4VAE, $h_2 = 8$}
  \end{subfigure}
  \hfill
    \begin{subfigure}[b]{0.24\textwidth}
    \includegraphics[width=\textwidth]{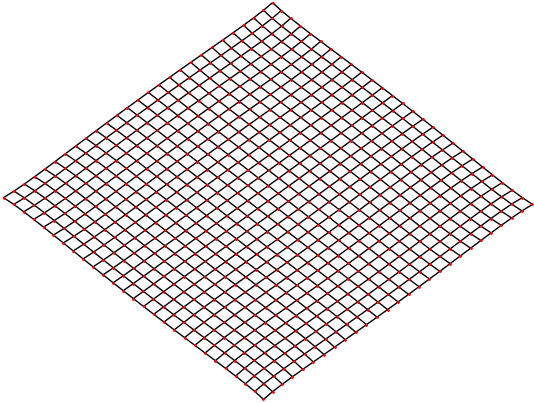}
    \caption{CVAE, $h_2 = 8$}
  \end{subfigure}
    \hfill
    \begin{subfigure}[b]{0.24\textwidth}
    \includegraphics[width=\textwidth]{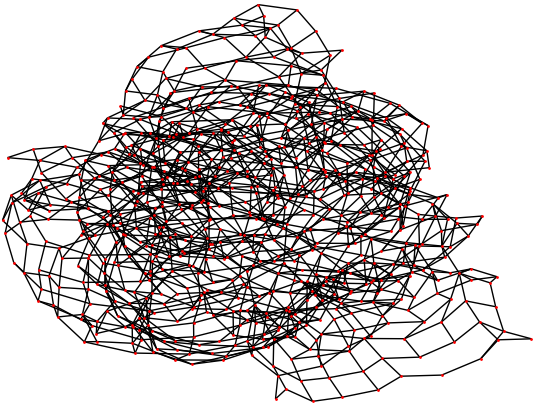}
    \caption{DVAE, $h_2 = 8$}
  \end{subfigure}
  \hfill
    \begin{subfigure}[b]{0.24\textwidth}
    \includegraphics[width=\textwidth]{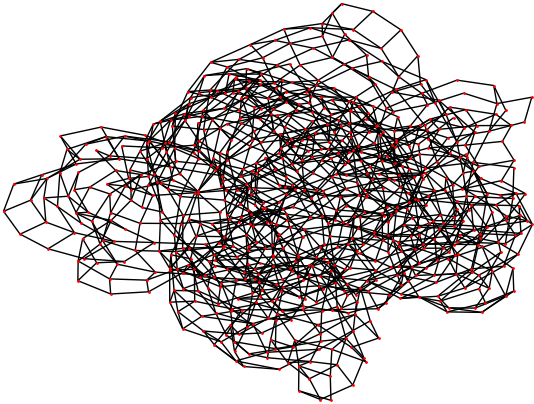}
    \caption{2VAE, $h_2 = 8$}
  \end{subfigure}
  \hfill
  \begin{subfigure}[b]{0.24\textwidth}
    \includegraphics[width=\textwidth]{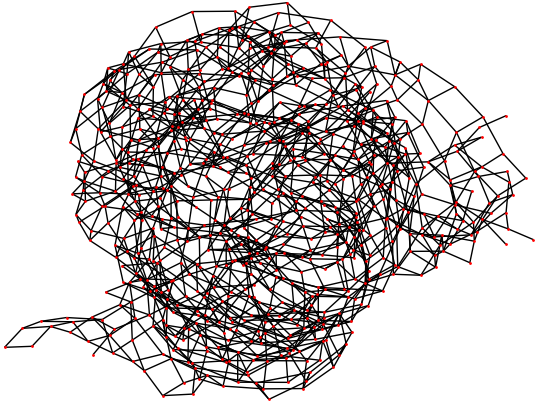}
    \caption{4VAE, $h_2 = 8$}
  \end{subfigure}
  \hfill
    \begin{subfigure}[b]{0.24\textwidth}
    \includegraphics[width=\textwidth]{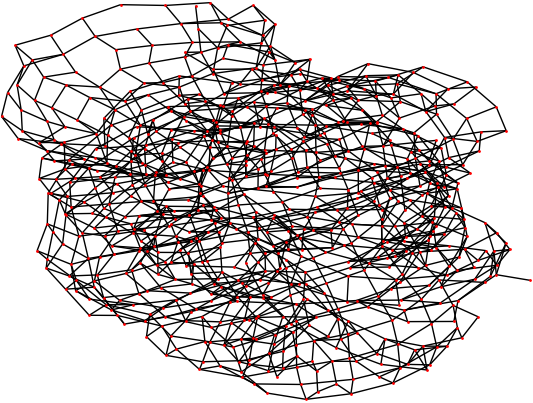}
    \caption{CVAE, $h_2 = 8$}
  \end{subfigure}
  \hfill
  \begin{subfigure}[b]{0.24\textwidth}
    \includegraphics[width=\textwidth]{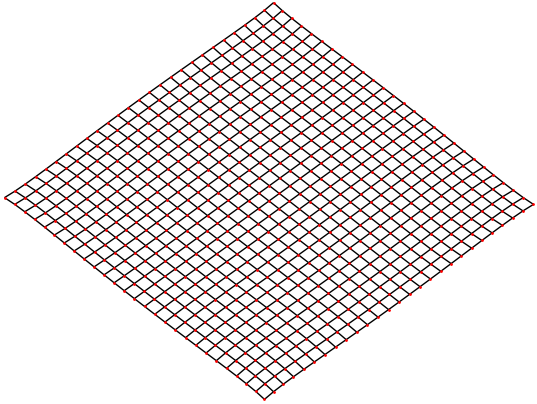}
    \caption{DVAE, $h_2 = 12$}
  \end{subfigure}
  \hfill
  \begin{subfigure}[b]{0.24\textwidth}
    \includegraphics[width=\textwidth]{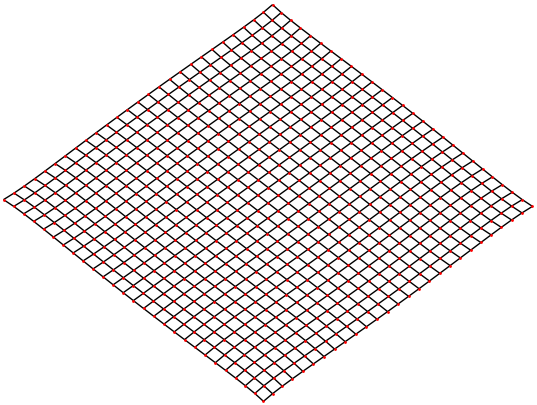}
    \caption{2VAE, $h_2 = 12$}
  \end{subfigure}
  \hfill
  \begin{subfigure}[b]{0.24\textwidth}
    \includegraphics[width=\textwidth]{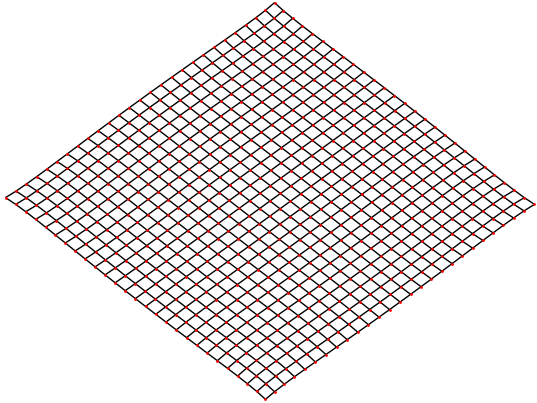}
    \caption{4VAE, $h_2 = 12$}
  \end{subfigure}
  \hfill
    \begin{subfigure}[b]{0.24\textwidth}
    \includegraphics[width=\textwidth]{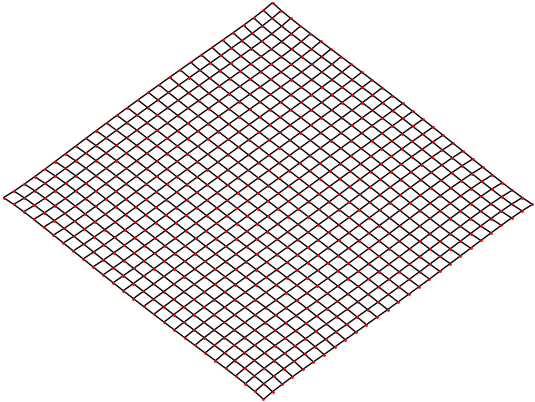}
    \caption{CVAE, $h_2 = 12$}
  \end{subfigure}
  \hfill
  \begin{subfigure}[b]{0.24\textwidth}
    \includegraphics[width=\textwidth]{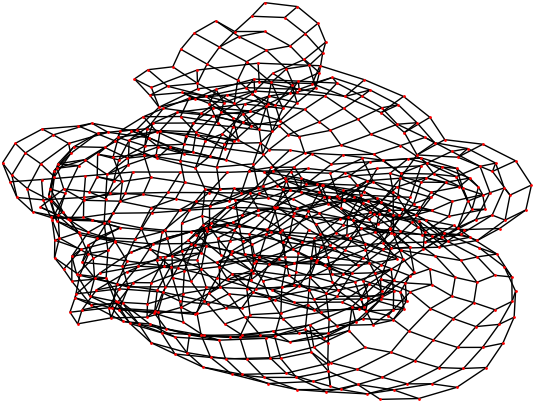}
    \caption{DVAE, $h_2 = 12$}
  \end{subfigure}
  \hfill
  \begin{subfigure}[b]{0.24\textwidth}
    \includegraphics[width=\textwidth]{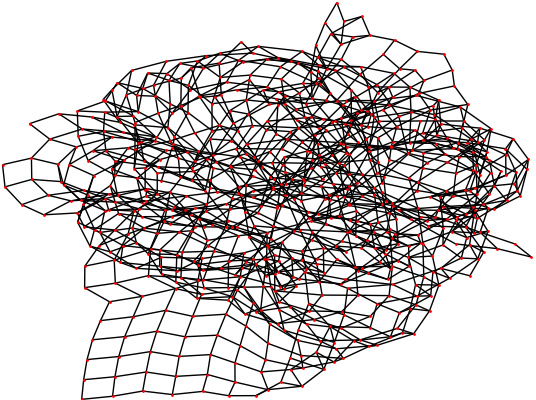}
    \caption{2VAE, $h_2 = 12$}
  \end{subfigure}
  \hfill
  \begin{subfigure}[b]{0.24\textwidth}
    \includegraphics[width=\textwidth]{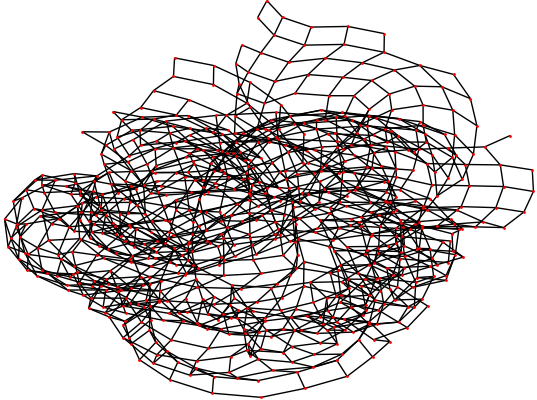}
    \caption{4VAE, $h_2 = 12$}
  \end{subfigure}
  \hfill
    \begin{subfigure}[b]{0.24\textwidth}
    \includegraphics[width=\textwidth]{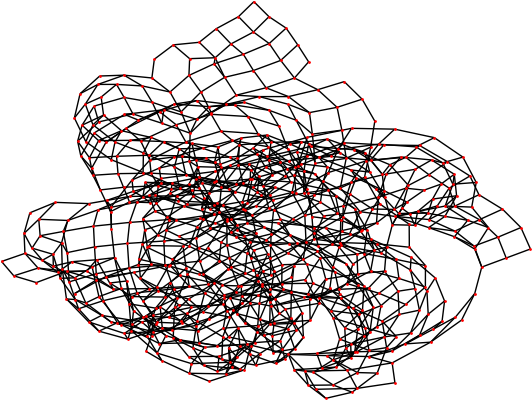}
    \caption{CVAE, $h_2 = 12$}
  \end{subfigure}
\caption{Visualized graph reconstructions for the long grid memorization setting in Figure $6$, main text. The deterministic sign decoder has flawlessly reconstructed the long grid for $h_2 \ge 8$, while the probabilistic decoding $\mathcal{B}\circ \sigma$ is less precise. Decreasing regularization  $\lambda$ allows backpropogation to increase the magnitude of the entries of $\tilde{\mathbf{A}}$, which pushes the probabilistic reconstruction closer to the deterministic region (not shown). Odd rows are decoded deterministically, and even rows probabilistically.}
\label{lastplot}
\end{figure}

\clearpage
\section{Cora and CiteSeer Networks}\label{realworldgraphnetworks}

In this section, we benchmark the original and augmented shallow embedding schemes against the Cora and CiteSeer datasets (\citealt{CoraCiteSeer}). The hyperparameters used were $h_1 = 2h_2$, $\lambda = 10^{-7}$, $\gamma = 10^{-3}$ and training was done for 20000 epochs for Cora and for 10000 epochs for CiteSeer. The Latent Dimensions entries give the value of $h_2$ used to instantiate the shallow embedding architectures, and the vertical architecture rank entries provide the ranks of the adjacency matrices discovered after training completion. The architecture error entries compute $||\mathbf{A}-\mathbf{\hat{A}}||_{Frob}^2$, which corresponds to double counting any edges memorized as non-edges in the distilled representation, or any non-edges memorized as edges. 95\% confidence intervals are given over $5$ experiments for all entries. 

In our empirical evaluations, the augmented architectures demonstrated substantial enhancements in terms of robustness and efficacy in approximating the graph adjacencies. This observation underscores the potential of the augmented models in capturing intricate graph-based relationships more accurately than their conventional counterparts.

\begin{table}[ht]
\centering

\caption{Performance Comparison for Cora}
\label{Cora2to10}
\vskip 0.15in
\begin{center}
\begin{small}
\begin{sc}
\begin{tabular}{lcccccc}
\toprule
Latent Dim & 2 & 4 & 6 & 8 & 10 \\
\midrule
VAE Error ($\times 10^4$)    & 229 $\pm$ 236 & 532 $\pm$ 55 & 574 $\pm$ 17 & 559 $\pm$ 33 & 593 $\pm$ 36 \\
VAE Rank     & 1.8 $\pm$ 0.39 & 4 $\pm$ 0 & 6 $\pm$ 0 & 8 $\pm$ 0 & 9.6 $\pm$ 0.48 \\
DVAE Error   & (183 $\pm$ 219)$10^4$ & $\boldsymbol{10144 \pm 111}$ & 9638 $\pm$ 154 & 8678 $\pm$ 588 & 6876 $\pm$ 551 \\
DVAE Rank    & 1.8 $\pm$ 0.39 & 4 $\pm$ 0 & 5.8 $\pm$ 0.39 & 7.8 $\pm$ 0.39 & 10 $\pm$ 0 \\
CVAE Error   & $\boldsymbol{11435 \pm 420}$ & 10349 $\pm$ 147 & $\boldsymbol{9453 \pm 45}$ & $\boldsymbol{8140 \pm 141}$ & $\boldsymbol{6300 \pm 222}$ \\
CVAE Rank    & 2 $\pm$ 0 & 4 $\pm$ 0 & 6 $\pm$ 0 & 8 $\pm$ 0 & 10 $\pm$ 0 \\
\bottomrule
\end{tabular}
\end{sc}
\end{small}
\end{center}
\vskip -0.1in

\caption{Performance Comparison for Cora}
\label{Cora20to640}
\vskip 0.15in
\begin{center}
\begin{small}
\begin{sc}
\begin{tabular}{lccccccc}
\toprule
Latent Dim & 20 & 40 & 80 & 160 & 320 & 640 \\
\midrule
VAE Error ($\times 10^4$)    & 607 $\pm$ 21 & 594 $\pm$ 9 & 573 $\pm$ 7 & 525 $\pm$ 5 & 442 $\pm$ 27 & 337 $\pm$ 32 \\
VAE Rank     & 17.6 $\pm$ 1 & 31.8 $\pm$ 1.69 & 50.4 $\pm$ 1.33 & 74.6 $\pm$ 1.18 & 96 $\pm$ 17 & 146 $\pm$ 10 \\
DVAE Error   & 2328 $\pm$ 286 & 1203 $\pm$ 309 & 702 $\pm$ 71 & 450 $\pm$ 95 & 290 $\pm$ 107 & 167 $\pm$ 127 \\
DVAE Rank    & 19.8 $\pm$ 0.39 & 38 $\pm$ 1.64 & 68 $\pm$ 2.15 & 99 $\pm$ 2.7 & 133 $\pm$ 5.7 & 136 $\pm$ 17 \\
CVAE Error   & $\boldsymbol{1594 \pm 44}$ & $\boldsymbol{356 \pm 55}$ & $\boldsymbol{88 \pm 43}$ & $\boldsymbol{35 \pm 11}$ & $\boldsymbol{28 \pm 20}$ & $\boldsymbol{22 \pm 5}$ \\
CVAE Rank    & 20 $\pm$ 0 & 40 $\pm$ 0 & 80 $\pm$ 0 & 160 $\pm$ 0 & 320 $\pm$ 0 & 580 $\pm$ 15 \\
\bottomrule
\end{tabular}
\end{sc}
\end{small}
\end{center}
\vskip -0.1in

\caption{Performance Comparison for CiteSeer}
\label{CiteSeer2to10}
\vskip 0.15in
\begin{center}
\begin{small}
\begin{sc}
\begin{tabular}{lcccccc}
\toprule
Latent Dim & 2 & 4 & 6 & 8 & 10 \\
\midrule
GAE Error ($\times 10^4$)    & 179 $\pm$ 321 & 712 $\pm$ 249 & 862 $\pm$ 160 & 956 $\pm$ 17 & 964 $\pm$ 6 \\
GAE Rank     & 1.8 $\pm$ 0.39 & 3.8 $\pm$ 0.39 & 6 $\pm$ 0 & 7.8 $\pm$ 0.39 & 9.2 $\pm$ 0.73 \\
DGAE Error   & $(458 \pm 374)10^4$ & $\boldsymbol{8849 \pm 74}$ & $\boldsymbol{8218 \pm 114}$ & $\boldsymbol{7534 \pm 276}$ & 6600 $\pm$ 274 \\
DGAE Rank    & 2 $\pm$ 0 & 4 $\pm$ 0 & 6 $\pm$ 0 & 8 $\pm$ 0 & 9.8 $\pm$ 0.39 \\
CGAE Error   & $\boldsymbol{25142 \pm 12756}$ & 12060 $\pm$ 3134 & 8455 $\pm$ 159 & 7612 $\pm$ 108 & $\boldsymbol{6512 \pm 137}$ \\
CGAE Rank    & 2 $\pm$ 0 & 4 $\pm$ 0 & 6 $\pm$ 0 & 8 $\pm$ 0 & 10 $\pm$ 0 \\
\bottomrule
\end{tabular}
\end{sc}
\end{small}
\end{center}
\vskip -0.1in

\caption{Performance Comparison for CiteSeer}
\label{CiteSeer20to640}
\vskip 0.15in
\begin{center}
\begin{small}
\begin{sc}
\begin{tabular}{lccccccc}
\toprule
Latent Dim & 20 & 40 & 80 & 160 & 320 & 640 \\
\midrule
GAE Error ($\times 10^4$)    & 967 $\pm$ 10 & 947 $\pm$ 3 & 911 $\pm$ 13 & 868 $\pm$ 2 & 820 $\pm$ 4 & 744 $\pm$ 8 \\
GAE Rank     & 17.8 $\pm$ 0.73 & 33.2 $\pm$ 0.73 & 55.6 $\pm$ 1.71 & 67 $\pm$ 1.75 & 84.4 $\pm$ 1.47 & 110 $\pm$ 4.25 \\
DGAE Error   & 3909 $\pm$ 87 & 2416 $\pm$ 209 & 2968 $\pm$ 2221 & 1884 $\pm$ 330 & $\boldsymbol{1160 \pm 395}$ & 486 $\pm$ 594 \\
DGAE Rank    & 19.6 $\pm$ 0.48 & 35.8 $\pm$ 2.43 & 56.4 $\pm$ 4.75 & 66.2 $\pm$ 2.59 & 74 $\pm$ 4.25 & 96.2 $\pm$ 14 \\
CGAE Error   & $\boldsymbol{3120 \pm 124}$ & $\boldsymbol{1336 \pm 89}$ & $\boldsymbol{1179 \pm 1857}$ & $\boldsymbol{197 \pm 216}$ & 2177 $\pm$ 4153 & $\boldsymbol{69 \pm 44}$ \\
CGAE Rank    & 20 $\pm$ 0 & 40 $\pm$ 0 & 80 $\pm$ 0 & 160 $\pm$ 0 & 316 $\pm$ 6 & 544 $\pm$ 21 \\
\bottomrule
\end{tabular}
\end{sc}
\end{small}
\end{center}
\vskip -0.1in
\end{table}

\end{document}